\documentclass{article}
\usepackage{amsmath,amssymb}
\usepackage{booktabs}
\usepackage{cases}
\usepackage{dsfont}
\usepackage{color}
\usepackage{caption}
\usepackage{natbib}
\usepackage{graphicx}

\newcommand{\BlackBox}{\rule{1.5ex}{1.5ex}}  
\newenvironment{proof}{\par\noindent{\bf Proof\ }}{\hfill\BlackBox\\[2mm]}
\newtheorem{example}{Example}
\newtheorem{theorem}{Theorem}
\newtheorem{lemma}[theorem]{Lemma}

\newtheorem{remark}[theorem]{Remark}

\newtheorem{definition}[theorem]{Definition}

\newif\ifdetails
\detailstrue

\usepackage{pdfsync}
\usepackage{hyperref}
\usepackage{algorithm} 
\usepackage{algpseudocode} 

\allowdisplaybreaks

\def\ba#1\ea{\begin{align*}#1\end{align*}} 
\def\banum#1\eanum{\begin{align}#1\end{align}} 

\bibpunct{[}{]}{;}{n}{}{,}

\newcommand{\Ber}{\mathop{\mathrm{Ber}}}
\newcommand{\cXstrat}{\varphi^{(\cX)}}
\newcommand{\Kstrat}{\psi^{(K+1)}}

\newcommand{\hit}{(H_t,I_t) \in \cC(h,i)}
\newcommand{\his}{(H_s,I_s) \in \cC(h,i)}

\renewcommand{\tilde}{\widetilde}
\renewcommand{\phi}{\varphi}
\renewcommand{\epsilon}{\varepsilon}
\newcommand{\eps}{\epsilon}
\let\P\undefined
\newcommand{\P}{\mathbb{P}}

\newcommand{\E}{\mathbb{E}}
\newcommand{\EE}[1]{\E\left[#1\right]}

\newcommand{\cA}{\mathcal{A}}
\newcommand{\cB}{\mathcal{B}}
\newcommand{\cN}{\mathcal{N}}

\newcommand{\cF}{\mathcal{F}}
\newcommand{\cG}{\mathcal{G}}

\newcommand{\cX}{\mathcal{X}}
\newcommand{\cI}{\mathcal{I}}
\newcommand{\cJ}{\mathcal{J}}

\newcommand{\cH}{\mathcal{H}}
\newcommand{\cP}{\mathcal{P}}
\newcommand{\cT}{\mathcal{T}}
\newcommand{\cC}{\mathcal{C}}

\newcommand{\R}{\mathbb{R}}
\newcommand{\real}{\mathbb{R}}
\newcommand{\ra}{\rightarrow}

\newcommand{\wh}{\widehat}

\renewcommand{\hat}{\widehat}

\newcommand{\eqdef}{\stackrel{\rm\scriptsize def}{=}}

\newcommand{\oneb}[1]{\mathbb{I}_{#1}}
\newcommand{\one}[1]{\oneb{\{#1\}}}
\newcommand{\onel}[1]{\oneb{\bigl\{ #1 \bigr\} }}
\newcommand{\fhi}{f_{h,i}^*}
\renewcommand{\leq}{\leqslant}
\renewcommand{\geq}{\geqslant}
\renewcommand{\le}{\leqslant}
\renewcommand{\ge}{\geqslant}

\newcommand{\RKernel}{M}
\newcommand{\Probs}{{\cal M}_1}
\newcommand{\Bandit}{{\cal B}}
\newcommand{\arm}{x}
\newcommand{\Arms}{\cX}
\newcommand{\Arm}{X}
\newcommand{\Reward}{Y}
\newcommand{\reward}{y}
\newcommand{\meanpayoff}{f}

\newcommand{\beq}{\begin{equation}}
\newcommand{\eeq}{\end{equation}}

\newcommand{\beqa}{\begin{eqnarray}}
\newcommand{\eeqa}{\end{eqnarray}}

\newcommand{\beqan}{\begin{eqnarray*}}
\newcommand{\eeqan}{\end{eqnarray*}}

\newenvironment{assumption}{\vspace{0.25cm} \noindent \textbf{Assumptions.} \em }{\vspace{0.25cm}}

\newcounter{assumption}
\newcommand{\theassumptionletter}{A}
\renewcommand{\theassumption}{\theassumptionletter\arabic{assumption}}

\newenvironment{ass*}[1][]{\begin{trivlist}\item[] %
 {\bf Assumption\  #1} }{
 \ifvmode\smallskip\fi\end{trivlist}}
\newenvironment{asss*}[1][]{\begin{trivlist}\item[] %
 {\bf Assumptions\  #1} }{
 \ifvmode\smallskip\fi\end{trivlist}}
 \newcommand{\aref}[1]{\ref{#1}}

\newcounter{assumptionV}
\renewcommand{\theassumptionV}{\theassumptionletter\arabic{assumptionV}'}

\newcommand{\cset}[2]{\left\{\,#1\,:\,#2\,\right\}}

\newcommand{\diam}{\mathop{\mathrm{diam}}}

\newcommand{\wrt}{w.r.t.}

\begin{document}

\title{$\cX$--Armed Bandits}

\author{
S{\'e}bastien Bubeck\\
Sequel Project, INRIA Lille \\
{\tt sebastien.bubeck@inria.fr}
\\ \\
R\'emi Munos \\
Sequel Project, INRIA Lille \\
{\tt remi.munos@inria.fr}
\\ \\
Gilles Stoltz \\
Ecole Normale Sup{\'e}rieure\footnote{This research was carried out
within the INRIA project CLASSIC hosted by
Ecole normale sup{\'e}rieure and CNRS.}, CNRS \\
\& \\
HEC Paris, CNRS, \\
{\tt gilles.stoltz@ens.fr}
\\ \\
Csaba Szepesv\'ari \\
University of Alberta, Department of Computing Science \\
{\tt szepesva@cs.ualberta.ca }}

\maketitle

\begin{abstract}
We consider a generalization of stochastic bandits where the set of arms, $\cX$,
is allowed to be a generic measurable space
and the mean-payoff function is ``locally Lipschitz'' with respect to a
dissimilarity function that is known to the decision maker.
Under this condition we construct an arm selection policy, called HOO (hierarchical optimistic optimization),
with improved regret bounds compared to previous results for a large class of problems.
In particular, our results imply that if $\cX$ is the unit hypercube in a Euclidean space
and the mean-payoff function has a finite number of global maxima
around which the behavior of the function is locally continuous with a known smoothness degree, then the expected regret of HOO is bounded up to a logarithmic factor by $\sqrt{n}$, i.e.,
the rate of growth of the regret is independent of the dimension of the space.
We also prove the minimax optimality of our algorithm when the dissimilarity is a metric.
Our basic strategy has quadratic computational complexity as a function of the number of time steps
and does not rely on the doubling trick.
We also introduce a modified strategy, which relies on the doubling trick but runs in linearithmic time.
Both results are improvements with respect to previous approaches.
\end{abstract}

\section{Introduction}

In the classical stochastic bandit problem a gambler tries to maximize his revenue
by sequentially playing one of a finite number of slot machines
that are associated with initially unknown
(and potentially different) payoff distributions \citep{Rob52}.
Assuming old-fashioned slot machines, the gambler pulls
the arms of the machines one by one in a sequential manner,
simultaneously learning about the machines' payoff-distributions and gaining actual monetary reward.
Thus, in order to maximize his gain, the gambler must choose the next arm by
taking into consideration both the urgency of gaining reward (``exploitation'')
and acquiring new information (``exploration'').

Maximizing the total cumulative payoff is equivalent to minimizing
the (total) {\em regret}, i.e., minimizing the difference between the total cumulative payoff of the
gambler and the one of another clairvoyant gambler who chooses the arm with the best mean-payoff in every
round. The quality of the gambler's strategy can be characterized as the rate of growth of his expected
regret with time. In particular, if this rate of growth is sublinear, the gambler in the long run plays
as well as the clairvoyant gambler. In this case the gambler's strategy is called Hannan consistent.

Bandit problems have been studied in the Bayesian framework \citep{gittins89}, as well as in the frequentist
parametric \citep{LR85,Agr95} and non-parametric settings \citep{ACF02}, and even
in non-stochastic scenarios \cite{ACFS02,CL06}. While in the Bayesian case the question is
whether the optimal actions can be computed efficiently, in the frequentist case the question is
how to achieve low rate of growth of the regret in the lack of prior information, i.e., it is a statistical question.
In this paper we consider the stochastic, frequentist, non-parametric setting.

Although the first papers studied bandits with a finite number of arms, researchers have soon realized
that bandits with infinitely many arms are also interesting,
as well as practically significant. One particularly important case is when the arms are identified
by a finite number of continuous-valued parameters,
resulting in {\em online optimization} problems over continuous finite-dimensional spaces.
Such problems are ubiquitous to operations research and control.
Examples are ``pricing a new product with uncertain demand in order to maximize revenue,
controlling the transmission power of a wireless communication system in a noisy
channel to maximize the number of bits transmitted per unit of power, and calibrating the temperature
or levels of other inputs to a reaction so as to maximize the yield of a chemical process'' \citep{Cop04}.
Other examples are optimizing parameters of schedules, rotational systems, traffic networks or online parameter tuning of numerical methods.
During the last decades numerous authors have investigated  such ``continuum-armed'' bandit
problems \citep{Agr95b,Kle04,AOS07,KSU08,Cop04}. A special case of interest, which forms a bridge between
the case of a finite number of arms and the continuum-armed setting, is formed by
bandit linear optimization, see \citep{AHA08} and the references therein.

In many of the above-mentioned problems, however, the natural domain of some of the optimization parameters
is a discrete set, while other parameters are still continuous-valued. For example, in the pricing problem
different product lines could also be tested while tuning the price, or in the case of transmission power
control different protocols could be tested while optimizing the power.
In other problems, such as in online sequential search, the parameter-vector to be optimized is an
infinite sequence over a finite alphabet \citep{CM07,BM10}.

The motivation for this paper is to handle all these various cases in a unified framework.
More precisely, we consider a general setting that allows us to study bandits with almost no restriction on the set of arms.
In particular, we allow the set of arms to be an arbitrary measurable space.
Since we allow non-denumerable sets, we shall assume that the gambler has some knowledge
about the behavior of the mean-payoff function (in terms of its local regularity around its maxima, roughly speaking).
This is because when the set of arms is uncountably infinite
and absolutely no assumptions are made on the payoff function, it is impossible to construct a strategy
that simultaneously achieves sublinear regret for all bandits problems (see, e.g., \cite[Corollary~4]{BMS10}).
When the set of arms is a metric space (possibly with the power of the continuum) previous works have assumed
either the global smoothness of the payoff function \citep{Agr95b,Kle04,KSU08,Cop04}
or local smoothness in the vicinity of the maxima \cite{AOS07}.
Here, smoothness means that the payoff function is either Lipschitz or H{\"o}lder continuous (locally or globally).
These smoothness assumptions are indeed
reasonable in many practical problems of interest.

In this paper, we assume that there exists a dissimilarity function that constrains the behavior of the mean-payoff function,
where a dissimilarity function is a measure of the discrepancy between two arms that is neither symmetric, nor reflexive,
nor satisfies the triangle inequality.
(The same notion was introduced simultaneously and independently of us by \cite[Section~4.4]{KSU08ext} under the name ``quasi-distance.'')
In particular, the dissimilarity function is assumed to locally set a bound on the decrease of the
mean-payoff function at each of its global maxima.
We also assume that the decision maker can construct a recursive covering of the space of
arms in such a way that the diameters of the sets in the covering shrink at a known geometric rate when measured
with this dissimilarity.

\paragraph{Relation to the literature.}
Our work generalizes and improves previous works on continuum-armed bandits.

In particular, \citet{Kle04} and \citet{AOS07} focused on one-dimensional problems, while we allow general spaces.
In this sense, the closest work to the present contribution
is that of \citet{KSU08}, who considered generic metric spaces assuming that the mean-payoff function
is Lipschitz with respect to the (known) metric of the space; its full version \cite{KSU08ext}
relaxed this condition and only requires that the mean-payoff function is Lipschitz at some maximum with
respect to some (known) dissimilarity.\footnote{
The present paper paper is a
concurrent and independent work with respect to the paper of Kleinberg,
Slivkins, and Upfal~\cite{KSU08ext}. An extended abstract~\cite{KSU08} of the latter was published in May 2008 at STOC'08,
while the NIPS'08 version~\cite{BMSS09} of the present paper was submitted at the beginning of June 2008.
At that time, we were not aware of the existence of the full version~\cite{KSU08ext}, which was released in September 2008.
}
 \citet{KSU08ext} proposed a novel algorithm
that achieves essentially the best possible regret bound in a minimax sense with respect to the environments studied,
as well as a much better regret bound if the mean-payoff function has a small ``zooming dimension''.

Our contribution furthers these works in two ways:
\begin{description}
\item[\ (i)] our algorithms, motivated by the recent successful tree-based
optimization algorithms \cite{KS06,GWMT06,CM07}, are easy to implement; 
\item[(ii)] we show that a version of our main algorithm is able to exploit the local properties of the mean-payoff function at its maxima only, which, as far as we know, was not investigated in the approach of \citet{KSU08,KSU08ext}.
\end{description}

The precise discussion of the improvements (and drawbacks) with respect to the papers by \citet{KSU08,KSU08ext}
requires the introduction of somewhat extensive notations and is therefore deferred to Section~\ref{sec:disc}.
However, in a nutshell, the following can be said.

First, by resorting to a hierarchical approach, we are able to avoid the use of the doubling trick,
as well as the need for the (covering) oracle, both of which the so-called zooming algorithm of~\citet{KSU08} relies on.
This comes at the cost of slightly more restrictive assumptions on the mean-payoff function, as well as a more involved analysis.
Moreover, the oracle is replaced by an {\em a priori} choice of a covering tree.
In standard metric spaces, such as the Euclidean spaces, such trees are trivial to construct,
 though, in full generality they may be difficult to obtain when their construction must start from (say) a distance function only.
We also propose
a variant of our algorithm that has smaller computational complexity of order $n \ln n$ compared to the quadratic complexity
$n^2$ of our basic algorithm.
However, the cheaper algorithm requires the doubling trick to achieve an anytime guarantee (just like the zooming algorithm).

Second, we are also able to weaken our assumptions and to consider only properties of the mean-payoff function in
the neighborhoods of its maxima; this leads to
regret bounds scaling as $\tilde O \bigl( \sqrt{n} \bigr)$
\footnote{We write $u_n = \tilde O(v_n)$ when $u_n = O(v_n)$ up to a logarithmic factor.}
when, e.g., the space is the unit hypercube and the mean-payoff function
has a finite number of global maxima $x^*$ around which it is locally
equivalent to a function $\Arrowvert x-x^*\Arrowvert^\alpha$ with some known degree $\alpha>0$.
Thus, in this case,
we get the desirable property that the rate of growth of the regret is independent of the dimensionality of the input space.
(Comparable dimensionality-free rates are obtained under different assumptions in~\cite{KSU08ext}.)

Finally, in addition to the strong theoretical guarantees,
we expect our algorithm to work well in practice
since the algorithm is very close to the recent, empirically very successful tree-search methods
from the games and planning literature
\citep{GeSi08:ICML,GeSi08:AAAI,schadd2008addressing,ChaWiHeUiBo08,finnsson2008simulation}.

\paragraph{Outline.}
The outline of the paper is as follows:
\begin{enumerate}
\item In Section~\ref{sec:setup} we formalize the $\cX$--armed bandit problem.
\item In Section~\ref{sec:HOO} we describe the basic strategy proposed,
called HOO (\emph{hierarchical optimistic optimization}).
\item We present the main results in Section~\ref{sec:mainresults}.
We start by specifying and explaining our assumptions (Section~\ref{sec:AssHOO})
under which various regret bounds are proved.
Then we prove a distribution-dependent bound for the basic version of HOO (Section~\ref{sec:RegrHOO}).
A problem with the basic algorithm is that its computational cost increases quadratically with the number of time steps.
Assuming the knowledge of the horizon, we thus propose
a computationally more efficient variant of the basic algorithm, called {\em truncated HOO}
and prove that it enjoys a regret bound identical to the one of the basic version (Section~\ref{sec:runningtime}) while
its computational complexity is only log-linear in the number of time steps.
The first set of assumptions constrains the mean-payoff function everywhere.
A second set of assumptions is therefore presented that puts constraints on the mean-payoff function
only in a small vicinity of its global maxima; we then propose another algorithm,
called {\em local-HOO},
which is proven to enjoy a regret again essentially similar to the one of the basic version
(Section~\ref{sec:localHOO}).
Finally, we prove the minimax optimality of HOO in metric spaces (Section~\ref{sec:minimax}).
\item In Section~\ref{sec:disc} we compare the results of this paper with previous works.
\end{enumerate}

\section{Problem setup} \label{sec:setup}

A {\em stochastic bandit problem} $\Bandit$ is a pair $\Bandit=(\cX,\RKernel)$,
where $\cX$ is a measurable space of arms
and $\RKernel$ determines the distribution of rewards associated with each arm.
We say that $\RKernel$ is a {\em bandit environment} on $\cX$. Formally, $\RKernel$ is an mapping
$\cX \ra \Probs(\real)$, where $\Probs(\real)$ is the space of probability distributions over the reals.
The distribution assigned to arm $\arm\in \Arms$ is denoted by $\RKernel_\arm$.
We require that for each arm $\arm\in \Arms$,
the distribution $\RKernel_\arm$ admits a first-order moment; we then denote by $\meanpayoff(\arm)$ its expectation (``mean payoff''),
\[
\meanpayoff(\arm) = \int \reward \,\, d\RKernel_\arm(\reward)\,.
\]
The mean-payoff function $\meanpayoff$ thus defined is assumed to be measurable. For simplicity, we shall also assume
that all $M_x$ have bounded supports, included in some fixed bounded interval\footnote{More generally, our results would also hold
when the tails of the reward distributions are uniformly sub-Gaussian.
}, say, the unit interval $[0,1]$.
Then, $f$ also takes bounded values, in $[0,1]$.

A decision maker (the gambler of the introduction)
that interacts with a stochastic bandit problem $\Bandit$ plays a game at discrete time steps according to the following rules.
In the first round the decision maker can select an arm $\Arm_1\in \Arms$ and receives a reward $\Reward_1$ drawn at random from $\RKernel_{\Arm_1}$.
In round $n>1$ the decision maker can select an arm $\Arm_n\in \Arms$ based on the information available up to time $n$, {i.e.}, $(\Arm_1,\Reward_1,\ldots,\Arm_{n-1},\Reward_{n-1})$, and receives a reward $\Reward_n$ drawn from $\RKernel_{\Arm_n}$, independently of $(\Arm_1,\Reward_1,\ldots,\Arm_{n-1},\Reward_{n-1})$ given $\Arm_n$. Note that  a decision maker may randomize his choice, but can only use information available up to the point in time when the choice is made.

Formally, a {\em strategy of the decision maker} in this game (``bandit strategy'') can be described by an infinite sequence
of measurable mappings, $\phi = (\phi_1,\phi_2,\ldots)$, where $\phi_n$ maps the space of past observations,
\[
\cH_n = \bigl( \cX \times [0,1] \bigr)^{n-1},
\]
to the space of probability measures over $\Arms$.
By convention, $\phi_1$ does not take any argument.
A strategy is called {\em deterministic} if for every $n$, $\phi_n$ is a Dirac distribution.

The {goal of the decision maker} is to maximize his expected cumulative reward.
Equivalently, the goal can be expressed as minimizing the expected cumulative regret, which is defined as follows.
Let
\[
f^* = \sup_{x\in \cX} f(x)
\]
be the best expected payoff in a single round.
At round $n$, the {\em cumulative regret} of a
decision maker playing $\Bandit$
is
\[
\wh{R}_n = n\,f^* - \sum_{t=1}^n Y_t,
\]
i.e., the difference between the maximum expected payoff in $n$ rounds and the actual total payoff.
In the sequel, we shall restrict our attention to the expected cumulative regret,
which is defined as the expectation $\E[\wh{R}_n]$ of the cumulative regret $\wh{R}_n$.

Finally, we define the cumulative {\em pseudo-regret} as
\[
R_n = n\,f^* - \sum_{t=1}^n f(X_t)\, ,
\]
that is, the actual rewards used in the definition of the regret are replaced by the mean-payoffs of the arms pulled.
Since (by the tower rule)
\[
\E \bigl[ Y_t \bigr] = \E \bigl[ \EE{Y_t|X_t} \bigr] = \E \bigl[ f(X_t) \bigr]\, ,
\]
the expected values $\E[\wh{R}_n]$ of the cumulative regret and
$\E[R_n]$ of the cumulative pseudo-regret are the same.
Thus, we focus below on the study of the behavior of $\E \bigl[ R_n \bigr]$.

\begin{remark}
As it is argued in \cite{BMS10}, in many real-world problems, the decision maker is not interested in his cumulative regret but rather in
its simple regret. The latter can be defined as follows. After $n$ rounds of play in a stochastic bandit problem $\cB$, the decision maker is asked to make a recommendation $Z_n \in \cX$ based on the $n$ obtained rewards $Y_1,\ldots,Y_n$.
The simple regret of this recommendation equals
\[
r_n = f^* - f(Z_n)\,.
\]
In this paper we focus on the cumulative regret $R_n$, but all the results can be readily extended to the simple regret
by considering the recommendation $Z_n = X_{T_n}$, where $T_n$ is drawn uniformly at random in $\{1,\hdots,n\}$. Indeed, in this case,
$$\E \bigl[ r_n \bigr] \leq \frac{\E \bigl[ R_n \bigr]}{n}\,,$$
as is shown in~\cite[Section 3]{BMS10}.
\end{remark}

\section{The Hierarchical Optimistic Optimization (HOO) strategy} \label{sec:HOO}

The HOO strategy (cf.\ Algorithm~\ref{alg:hoo}) incrementally builds an estimate of the mean-payoff function $\meanpayoff$ over $\cX$.
The core idea (as in previous works) is to estimate $\meanpayoff$ precisely around its maxima, while estimating it loosely in
other parts of the space $\cX$. To implement this idea, HOO maintains a binary tree whose nodes are
associated with measurable regions of the arm-space $\cX$ such that the regions associated with nodes
deeper in the tree (further away from the root) represent increasingly smaller subsets of $\cX$.
The tree is built in an incremental manner. At each node of the tree, HOO stores some statistics based on the
information received in previous rounds. In particular, HOO keeps track of the number of times
a node was
traversed up to round $n$ and the corresponding empirical average
of the rewards received so far. Based on these,
HOO assigns an optimistic estimate (denoted by $B$) to the maximum mean-payoff associated with each node.
These estimates are then used to select the next node to ``play''.
This is done by traversing the tree, beginning from the root, and always following the node with the highest $B$--value
(cf.\ lines~\ref{algline:treetraversalstart}--\ref{algline:treetraversalend} of Algorithm~\ref{alg:hoo}).
Once a node is selected, a point in the region associated with it is chosen (line~\ref{algline:choice})
and is sent to the environment.
Based on the point selected and the received reward, the tree is updated (lines~\ref{algline:update1}--\ref{algline:update2}).

\bigskip
The tree of coverings which HOO needs to receive as an input
is an infinite binary tree whose nodes are associated with subsets of $\cX$.
The nodes in this tree are indexed by pairs of integers $(h,i)$;
node $(h,i)$ is located at depth $h\ge 0$ from the root.
The range of the second index, $i$, associated with nodes at depth $h$ is restricted by
$1\le i \le 2^h$.
Thus, the root node is denoted by $(0,1)$.
By convention, $(h+1,2i-1)$ and $(h+1,2i)$ are used to refer to the two children of the node $(h,i)$.
Let $\cP_{h,i}\subset \cX$ be the region associated with node $(h,i)$.
By assumption, these regions are measurable and must satisfy
 the constraints
\begin{subequations}
\begin{align}
\label{eq:coverings}
\cP_{0,1} &= \Arms\,,  &\mbox{} \\
\label{eq:coverings-c}
\cP_{h,i} &= \cP_{h+1,2i-1} \cup \cP_{h,2i}\,, \quad
	      & \mbox{for all} \ h\ge 0 \ \mbox{and} \ 1\le i \le 2^{h}.
\end{align}
\end{subequations}
As a corollary, the regions $\cP_{h,i}$ at any level $h \geq 0$ cover the space $\cX$,
\[
\cX = \bigcup_{i = 1}^{2^h} \, \cP_{h,i}\,,
\]
explaining the term ``tree of coverings''.

In the algorithm listing
the recursive computation of the $B$--values (lines~\ref{algline:bvalueupdate1}--\ref{algline:bvalueupdate2})
makes a local copy of the tree; of course, this part of the algorithm could be implemented in various other ways.
Other arbitrary choices in the algorithm as shown here are how tie breaking in the node selection part is done (lines~\ref{algline:tiebreak1}--\ref{algline:tiebreak2}), or how a point in the region associated with the selected node is chosen (line~\ref{algline:choice}).
We note in passing that implementing these differently would not change our theoretical results.

\label{sec:HOO.algo}
\begin{algorithm}[p]
\ \\

\noindent
\textbf{Parameters:} \ Two real numbers $\nu_1 > 0$ and $\rho \in (0,1)$,
a sequence $(\cP_{h,i})_{h \geq 0, 1 \leq i \leq 2^h}$ of subsets of $\cX$
satisfying the conditions~\eqref{eq:coverings} and~\eqref{eq:coverings-c}. \\

\textbf{Auxiliary function} \textsc{Leaf}($\cT$): \ outputs a leaf of $\cT$. \\

\noindent
\textbf{Initialization:} \ $\cT= \bigl\{ (0,1) \bigr\}$ and $B_{1,2}=B_{2,2}=+\infty$. \\

\begin{algorithmic}[1]
\For{$n = 1, 2, \ldots$} \Comment{Strategy HOO in round $n\ge 1$}
\State $(h,i) \gets (0,1)$ \Comment{Start at the root}
\State $P \gets \{ (h,i) \}$ \Comment{$P$ stores the path traversed in the tree}
\While{$(h,i) \in \cT$} \Comment{Search the tree $\cT$} \label{algline:treetraversalstart}
	\If{$B_{h+1,2i-1}>B_{h+1,2i}$ } \Comment{Select the ``more promising'' child}
		\State $(h,i) \gets (h+1,2i-1)$
	\ElsIf{$B_{h+1,2i-1}<B_{h+1,2i}$ }
		\State $(h,i) \gets (h+1,2i)$
	\Else \Comment{Tie-breaking rule} \label{algline:tiebreak1}
		\State $Z \sim {\rm Ber}(0.5)$ \Comment{e.g., choose a child at random}
		\State $(h,i) \gets (h+1,2i-Z)$
	\EndIf \label{algline:tiebreak2}
	\State $P \gets P \cup \{ (h,i) \}$
\EndWhile \label{algline:treetraversalend}
\State $(H,I) \gets (h,i)$ \Comment{The selected node} \label{algline:choicenode}
\State \label{algline:choice}
Choose arm $\Arm$ in $\cP_{H,I}$ and play it \Comment{Arbitrary selection of an arm}
\State Receive corresponding reward $\Reward$
\State $\cT \gets \cT \cup \{(H,I)\}$ \Comment{Extend the tree} \label{algline:update1}
\ForAll{$(h,i)\in P$ } \Comment{Update the statistics $T$ and $\hat{\mu}$ stored in the path}
	\State $T_{h,i}\gets T_{h,i}+1$ \Comment{Increment the counter of node $(h,i)$}
	\State $\hat{\mu}_{h,i} \gets \bigl(1-1/T_{h,i}\bigr)\hat{\mu}_{h,i} + Y/T_{h,i}$
\Comment{Update the mean $\hat{\mu}_{h,i}$ of node $(h,i)$}
\EndFor
\ForAll{$(h,i)\in \cT$ } \Comment{Update the statistics $U$ stored in the tree}
	\State $U_{h,i}\gets \hat{\mu}_{h,i} + \sqrt{(2 \ln n)/{T_{h,i}}} + \nu_1 \rho^h$ \Comment{Update the $U$--value of node $(h,i)$}
\EndFor
\State $B_{H+1,2I-1} \gets + \infty$ \Comment{$B$--values of the children of the new leaf}
\State $B_{H+1,2I} \gets + \infty$
\State $\cT' \gets \cT$ \Comment{Local copy of the current tree $\cT$} \label{algline:bvalueupdate1}
\While{$\cT' \ne \bigl\{ (0,1) \bigr\}$} \Comment{Backward computation of the $B$--values}
    \State $(h,i) \gets \mbox{\textsc{Leaf}}(\cT')$ \Comment{Take any remaining leaf}
 	\State $B_{h,i} \gets \min\Bigl\{ U_{h,i}, \,\max\bigl\{ B_{h+1,2i-1}, B_{h+1,2i} \bigr\} \Bigr\}$ \Comment{Backward computation}
    \State $\cT' \gets \cT' \setminus \bigl\{ (h,i) \bigr\}$ \Comment{Drop updated leaf $(h,i)$}
\EndWhile \label{algline:update2} \label{algline:bvalueupdate2}
\EndFor
\end{algorithmic}
\caption{\quad The HOO strategy} \label{alg:hoo}
\end{algorithm}

\bigskip
To facilitate the formal study of the algorithm, we shall need some more notation.
In particular, we shall introduce time-indexed versions
($\cT_n$, $(H_n,I_n)$, $X_n$, $Y_n$, $\wh{\mu}_{h,i}(n)$, etc.) of the quantities
used by the algorithm.
The convention used is that the indexation by $n$ is used to indicate the value
taken at the end of the $n^{\rm th}$ round.

In particular, $\cT_n$ is used to denote the finite subtree stored by the algorithm at the end of round $n$.
Thus, the initial tree is $\cT_0 = \{(0,1)\}$ and it is expanded round after round as
\[
\cT_{n} = \cT_{n-1}  \cup \{(H_n,I_n)\}\,,
\]
where $(H_n,I_n)$ is the node selected in line~\ref{algline:choicenode}.
We call $(H_n,I_n)$ {\em the node played in round $n$}.
We use $X_n$ to denote the point selected by HOO in the region associated with the node played in round $n$,
 while $Y_n$ denotes the received reward.

Node selection works by comparing $B$--values and always choosing the node with the highest $B$--value.
The $B$--value, $B_{h,i}(n)$, at node $(h,i)$ by the end of round $n$  is an estimated upper bound on the mean-payoff function
at node $(h,i)$.
To define it we first need to introduce the average of the rewards received in rounds
when some descendant of node $(h,i)$ was chosen (by convention, each node is a
descendant of itself):
\[
\wh{\mu}_{h,i}(n) = \frac{1}{T_{h,i}(n)} \, \sum_{t=1}^{n} \, Y_t \, \oneb{ \{ \hit \} }\,.
\]
Here, $\cC(h,i)$ denotes the set of all descendants of a node $(h,i)$ in the infinite tree,
\[
\cC(h,i) = \bigl\{ (h,i) \bigr\} \cup \cC(h+1,2i-1) \cup \cC(h+1,2i)\,,
\]
and
$T_{h,i}(n)$ is the number of times a descendant of $(h,i)$ is played up to and including round $n$, that is,
\[
T_{h,i}(n) = \sum_{t=1}^{n} \oneb{ \{ \hit \} }\,.
\]
A key quantity determining $B_{h,i}(n)$ is $U_{h,i}(n)$,
an initial estimate of the maximum of the mean-payoff
function in the region $\cP_{h,i}$ associated with node $(h,i)$:
\begin{equation}
\label{eq:U}
U_{h,i}(n) =
\begin{cases}
\wh{\mu}_{h,i}(n) + \displaystyle{\sqrt{\frac{2 \ln n}{T_{h,i}(n)}}} + \nu_1 \rho^h, & \text{if } T_{h,i}(n)>0; \vspace{.15cm} \\
+\infty, & \text{otherwise}.
\end{cases}
\end{equation}
In the expression corresponding to the case $T_{h,i}(n)>0$,
 the first term added to the average of rewards
 accounts for the uncertainty arising from the randomness of the rewards
 that the average is based on,
 while the second term, $\nu_1 \rho^h$,
 accounts for the maximum possible variation of the mean-payoff function over the region $\cP_{h,i}$.
The actual bound on the maxima used in HOO is defined recursively by
\[
B_{h,i}(n) =
\begin{cases}
\min\Bigl\{ U_{h,i}(n), \,\max\bigl\{ B_{h+1,2i-1}(n), B_{h+1,2i}(n) \bigr\} \Bigr\},
& \text{if } (h,i) \in \cT_{n};\\
+\infty, & \text{otherwise}.
\end{cases}\nonumber
\]
The role of $B_{h,i}(n)$ is to put a tight, optimistic, high-probability
upper bound on the best mean-payoff that can be achieved in the region $\cP_{h,i}$.
By assumption, $\cP_{h,i} = \cP_{h+1,2i-1} \cup \cP_{h+1,2i}$.
Thus, assuming that $B_{h+1,2i-1}(n)$ (resp., $B_{h+1,2i}(n)$) is a valid upper bound
for region $\cP_{h+1,2i-1}$ (resp., $\cP_{h+1,2i}$),
we see that $\max\bigl\{ B_{h+1,2i-1}(n), B_{h+1,2i}(n) \bigr\} $ must be a valid upper bound
for region $\cP_{h,i}$. Since $U_{h,i}(n)$ is another valid upper bound for region $\cP_{h,i}$,
we get a tighter (less overoptimistic) upper bound by taking the minimum of these bounds.

Obviously, for leafs $(h,i)$ of the tree $\cT_n$, one has $B_{h,i}(n) = U_{h,i}(n)$,
while close to the root one may expect that $B_{h,i}(n) < U_{h,i}(n)$; that is,
the upper bounds close to the root are expected to be less biased than the
ones associated with nodes farther away from the root.

\bigskip
Note that at the beginning of round $n$, the algorithm uses $B_{h,i}(n-1)$ to select the node $(H_n,I_n)$
to be played (since $B_{h,i}(n)$ will only be available at the end of round $n$).
It does so by following a path from the root node to an inner node with only one child or a leaf
and finally considering a child $(H_n,I_n)$ of the latter; at each node of the path, the child with
highest $B$--value is chosen, till the node $(H_n,I_n)$ with infinite $B$--value is reached.

\paragraph{Illustrations.}
Figure~\ref{fig:tree} illustrates the computation done by HOO in round $n$, as well as the correspondence between
the nodes of the tree constructed by the algorithm and their associated regions.
Figure~\ref{fig:tree2} shows trees built by running HOO for a specific environment.
\begin{figure}[t]
\begin{center}
\includegraphics[scale=0.4]{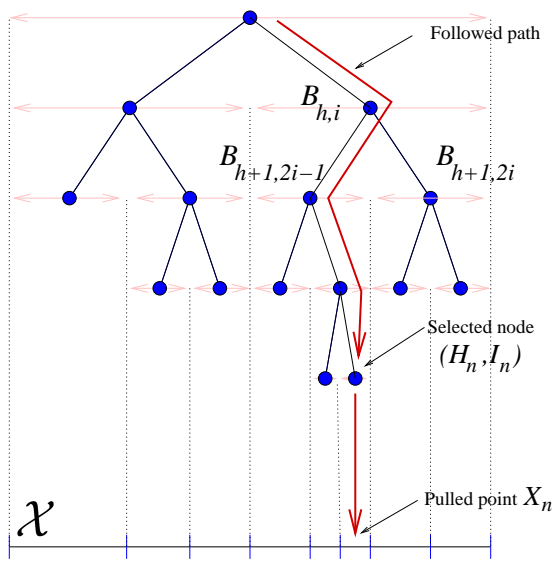}
\caption{Illustration of the node selection procedure in round $n$. The tree represents $\cT_{n}$.
In the illustration, $B_{h+1,2i-1}(n-1) > B_{h+1,2i}(n-1)$, therefore, the selected path included the node $(h+1,2i-1)$
rather than the node $(h+1,2i)$.}
\label{fig:tree}
\end{center}
\end{figure}
\begin{figure}[t]
\begin{center}
\includegraphics[scale=0.4]{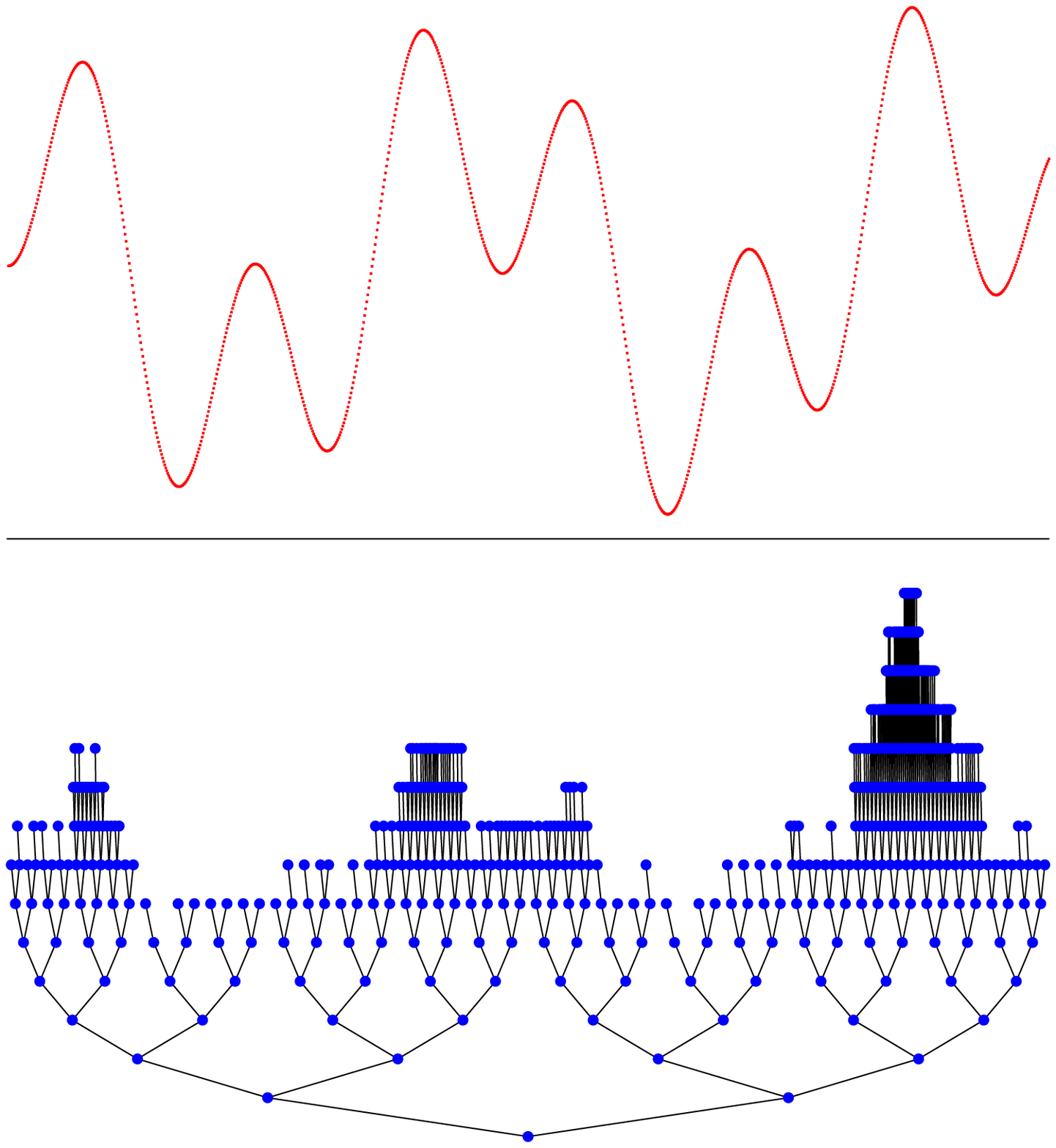} \hspace{2cm}
\includegraphics[scale=0.4]{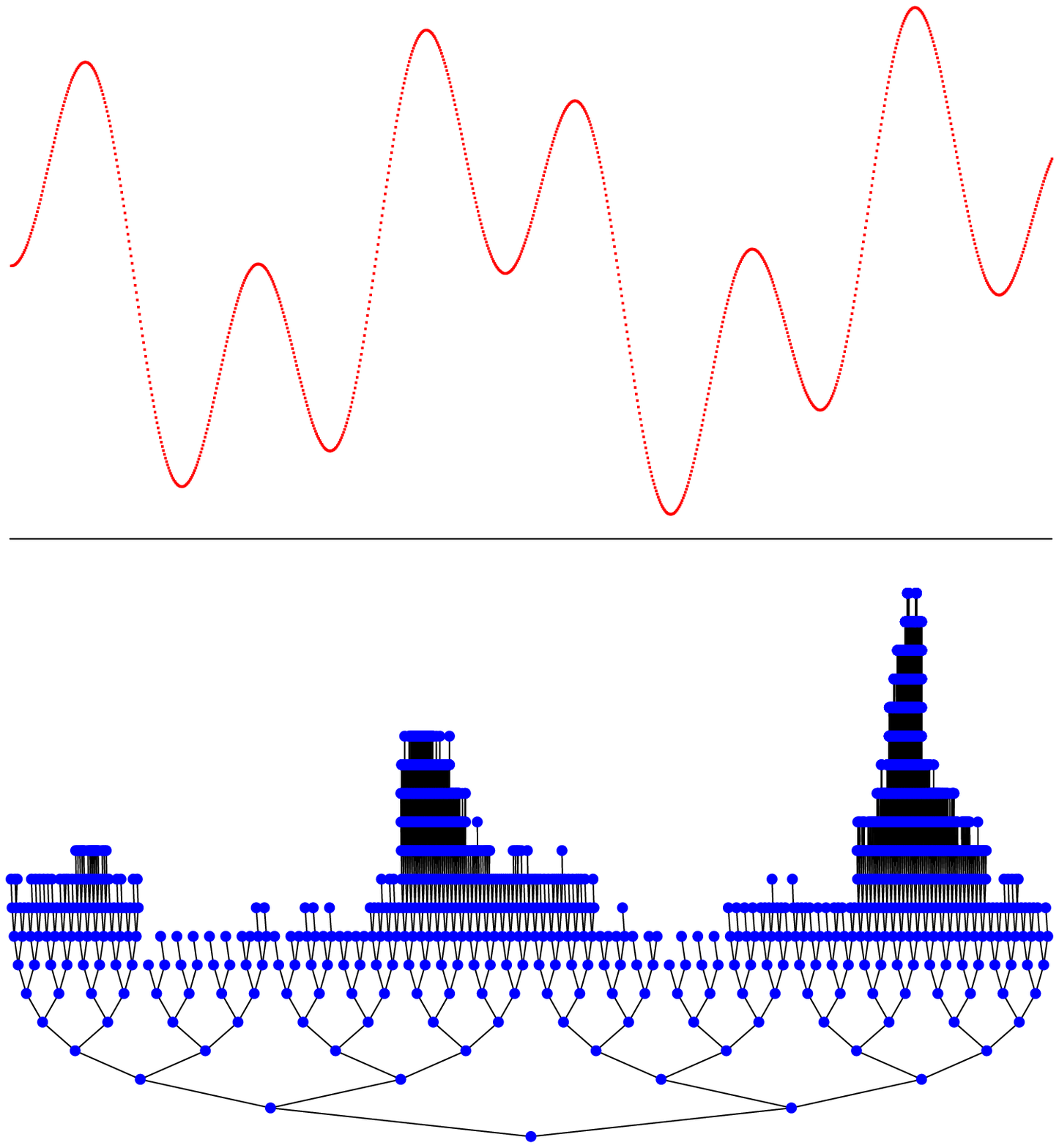}
\caption{The trees (bottom figures) built by HOO after 1,000 (left) and 10,000 (right) rounds.
The mean-payoff function (shown in the top part of the figure) is
$x \in [0,1] \, \longmapsto \, 1/2 \bigl( \sin(13 x) \sin(27 x) +1 \bigr)$;
the corresponding payoffs are Bernoulli-distributed.
The inputs of HOO are as follows:  the tree of coverings is formed by all dyadic intervals,
$\nu_1=1$ and $\rho=1/2$.
The tie-breaking rule is to choose a child at random (as shown in the Algorithm~\ref{alg:hoo}),
while the points in $\cX$ to be played are chosen as the centers of the dyadic intervals.
Note that the tree is extensively refined where the mean-payoff function is near-optimal,
while it is much less developed in other regions.
}
\label{fig:tree2}
\end{center}
\end{figure}

\paragraph{Computational complexity.}
At the end of round $n$, the size of the active tree $\cT_{n}$ is at most $n$, making
the storage requirements of HOO linear in $n$.
In addition, the statistics and $B$--values of all nodes in the active tree
need to be updated, which thus takes  time $O(n)$.
HOO runs in time $O(n)$ at each round $n$, making the algorithm's total running time up to round $n$ quadratic in $n$.
In Section~\ref{sec:runningtime} we
modify HOO so that if the time horizon $n_0$ is known in advance, the total running time is $O(n_0 \ln n_0)$,
while the modified algorithm will be shown to enjoy essentially the same regret bound as the original version.

\section{Main results} \label{sec:mainresults}

We start by describing and commenting on the assumptions that we need to analyze
the regret of HOO. This is followed by stating the first upper bound, followed by some
improvements on the basic algorithm. The section is finished by the statement of
our results on the minimax optimality of HOO.

\subsection{Assumptions}
\label{sec:AssHOO}

The main assumption will concern the ``smoothness'' of the mean-payoff function.
However, somewhat unconventionally, we shall use a notion of smoothness that is built
around dissimilarity functions rather than distances,
allowing us to deal with function classes of
highly different smoothness degrees in a unified manner.
Before stating our smoothness assumptions, we
define the notion of a dissimilarity function and some associated concepts.

\begin{definition}[Dissimilarity]
A {\em dissimilarity} $\ell$ over $\cX$ is a non-negative mapping $\ell:\cX^2 \to \R$
satisfying $\ell(x,x) = 0$ for all $x \in \cX$.
\end{definition}
Given a dissimilarity $\ell$, the {\em diameter} of a subset $A$ of $\cX$ as measured by $\ell$ is defined by
\[
\diam(A) = \sup_{x,y \in A} \ell(x,y)\,,
\]
while the {\em $\ell$--open ball} of $\cX$ with radius $\eps>0$ and center $x \in \cX$
is defined by
\[
\cB(x,\eps) = \cset{ y\in \cX }{ \ell(x,y)<\eps }\,.
\]
Note that the dissimilarity $\ell$ is only used in the theoretical analysis of HOO;
the algorithm does not require $\ell$ as an explicit input. However, when choosing its
parameters (the tree of coverings and the real numbers $\nu_1 > 0$ and $\rho < 1$) for the (set of) two assumptions below to be satisfied,
the user of the algorithm probably has in mind a given dissimilarity.

However, it is also natural to wonder what is the class of functions for
which the algorithm (given a fixed tree) can achieve non-trivial regret bounds;
a similar question for regression was investigated e.g., by \citet{Yang07}.
We shall indicate below how to construct a subset of such a class,
right after stating our assumptions connecting the tree, the dissimilarity, and the environment (the mean-payoff function).
Of these, Assumption~\ref{ass:weaklip} will be interpreted, discussed, and equivalently reformulated
below into~(\ref{eq:A2eqv}), a form that might be more intuitive.
The form (\ref{eq:weak.Lipschitz2}) stated below will turn out to be the most useful one in the proofs.

\begin{asss*}
Given the parameters of HOO, that is, the real numbers $\nu_1 > 0$ and $\rho \in (0,1)$ and the tree of coverings $(\cP_{h,i})$,
there exists a dissimilarity function $\ell$ such that the following two assumptions are satisfied.
\begin{enumerate}
\refstepcounter{assumption}
\item[\theassumption.]
  \label{ass:shrinking}
  There exists $\nu_2 > 0$ such that
  for all integers $h \geq 0$,
  \begin{enumerate}
    \item $\diam(\cP_{h,i}) \le \nu_1 \rho^h$ for all $i = 1,\ldots,2^h$;
    \item for all $i = 1,\ldots,2^h$, there exists $\arm^\circ_{h,i}\in \cP_{h,i}$
     such that
     \[
     \cB_{h,i} \eqdef \cB \bigl(\arm^\circ_{h,i},\,\nu_2 \rho^h\bigr) \subset \cP_{h,i}~;
     \]
    \item $\cB_{h,i} \cap \cB_{h,j} = \emptyset$ for all $1\le i<j\le 2^h$.
  \end{enumerate}
\refstepcounter{assumption}
\setcounter{assumptionV}{\value{assumption}}
\item[\theassumption.]
\label{ass:weaklip}
The mean-payoff function $f$ satisfies
that for all $x,y\in \cX$,
\begin{equation}
\label{eq:weak.Lipschitz2}
f^*-f(y) \leq  f^*-f(x) + \max\bigl\{ f^*-f(x), \, \ell(x,y) \bigr\}\,.
\end{equation}
\end{enumerate}
\end{asss*}

We show next how a tree induces in a natural way first a dissimilarity and then a class of environments.
For this, we need to assume that the tree of coverings $(\cP_{h,i})$ --in addition to~\eqref{eq:coverings} and~\eqref{eq:coverings-c}--
is such that the subsets $\cP_{h,i}$ and $\cP_{h,j}$ are disjoint whenever $1\le i < j \le 2^h$ and that none of them is empty.
Then, each $x \in \cX$ corresponds to a unique path in the tree, which can be represented as an infinite binary sequence
$x_0 x_1 x_2 \ldots$, where
\begin{eqnarray*}
x_0 & = & \onel{ x \in \cP_{1,1+1} }\,, \\
x_{1} & = & \onel{ x\in \cP_{2,1+(2x_0+1)}}\,, \\
x_{2} & = & \onel{ x\in \cP_{3,1+(4x_0+2x_1+1)}}\,, \\
\ldots
\end{eqnarray*}
For points $x,y \in \cX$ with respective representations $x_0 x_1 \ldots$ and $y_0 y_1 \ldots$, we let
\[
\ell(x,y) = (1-\rho) \nu_1 \, \sum_{h=0}^\infty \one{x_h\not=y_h} \rho^{h}\,.
\]
It is not hard to see that this dissimilarity satisfies~\ref{ass:shrinking}.
Thus, the associated class of environments ${\cal C}$ is formed by those with mean-payoff functions
satisfying~\ref{ass:weaklip} with the so-defined dissimilarity.
This is a ``natural class'' underlying the tree for which our tree-based algorithm
can achieve non-trivial regret. (However, we do not know if this is the largest such class.)
\newline

In general, Assumption~\ref{ass:shrinking} ensures that the regions in the
tree of coverings $(\cP_{h,i})$ shrink exactly at a geometric rate.
The following example shows how to satisfy~\ref{ass:shrinking} when the domain $\Arms$
is a $D$--dimensional hyper-rectangle and
the dissimilarity is some positive power of the Euclidean (or supremum) norm.

\begin{example} \label{ex:1}
Assume that $\cX$ is a $D$-dimension hyper-rectangle and consider the dissimilarity
$\ell(x,y) = b \|x-y\|_2^a$, where $a > 0$ and $b>0$ are real numbers and $\|\cdot\|_2$ is the Euclidean norm.
Define the tree of coverings $(\cP_{h,i})$ in the following inductive way: let $\cP_{0,1} = \cX$.
Given a node $\cP_{h,i}$, let $\cP_{h+1,2i-1}$ and $\cP_{h+1,2i}$ be obtained from the hyper-rectangle
$\cP_{h,i}$ by splitting it in the middle along its longest side
(ties can be broken arbitrarily).

We now argue that Assumption~\ref{ass:shrinking} is satisfied.
With no loss of generality we take $\cX = [0,1]^D$. Then, for all integers $u \geq 0$ and $0 \leq k \leq D-1$,
\[
\diam(\cP_{u D + k, 1}) = b \left( \frac{1}{2^u} \sqrt{D - \frac{3}{4}\,k}\right)^a \leq b \left( \frac{\sqrt{D}}{2^u} \right)^{\! a}\,.
\]
It is now easy to see that
Assumption~\ref{ass:shrinking} is satisfied for the indicated dissimilarity,
e.g., with the choice of the parameters
$\rho = 2^{-a/D}$ and $\nu_1 = b \, \bigl(2 \sqrt{D} \bigr)^{a}$ for HOO,
and the value $\nu_2 = b / 2^{a}$.
\end{example}

\begin{example} \label{ex:1ctd}
In the same setting, with the same tree of coverings $(\cP_{h,i})$ over $\cX = [0,1]^D$, but
now with the dissimilarity $\ell(x,y) = b \|x - y\|_{\infty}^a$, we get that
for all integers $u \geq 0$ and $0 \leq k \leq D-1$,
\[
\diam(\cP_{u D + k, 1}) = b \left( \frac{1}{2^u} \right)^a\,.
\]
This time, Assumption~\ref{ass:shrinking} is satisfied,
e.g., with the choice of the parameters
$\rho = 2^{-a/D}$ and $\nu_1 = b \, 2^a$ for HOO,
and the value $\nu_2 = b / 2^{a}$.
\end{example}

The second assumption, \ref{ass:weaklip}, concerns the environment;
when Assumption~\ref{ass:weaklip} is satisfied,
we say that $f$ is {\em weakly Lipschitz} with respect to (w.r.t.) $\ell$.
The choice of this terminology follows from the fact
that if $f$ is $1$--Lipschitz \wrt\ $\ell$, i.e., for all $x,y \in \cX$,
one has $|f(x)-f(y)|\leq \ell(x,y)$, then it is also weakly Lipschitz \wrt\ $\ell$.

On the other hand, weak Lipschitzness is a milder requirement. It
implies local (one-sided) $1$--Lipschitzness at any  global
maximum, since at any arm $x^*$ such that $f(x^*) = f^*$,
the criterion (\ref{eq:weak.Lipschitz2}) rewrites to $f(x^*) - f(y) \leq \ell(x^*,y)$.
In the vicinity of other arms $x$, the constraint is milder as the arm $x$ gets worse
(as $f^*-f(x)$ increases) since the condition (\ref{eq:weak.Lipschitz2}) rewrites to
\begin{equation}
\label{eq:A2eqv}
\forall \, y \in \cX, \qquad f(x) - f(y) \leq \max\bigl\{ f^*-f(x), \, \ell(x,y) \bigr\}\,.
\end{equation}

Here is another interpretation of these two facts; it will be useful when considering
local assumptions in Section~\ref{sec:localHOO} (a weaker set of assumptions).
First, concerning the behavior around global maxima, Assumption~\ref{ass:weaklip} implies that
for any set $\cA \subset \cX$ with $\sup_{x \in \cA} f(x) = f^*$,
\begin{equation} \label{eq:WL1}
f^*-\inf_{x \in \cA} f(x) \leq  \diam(\cA).
\end{equation}
Second, it can be seen that Assumption~\ref{ass:weaklip} is equivalent\footnote{That
Assumption~\ref{ass:weaklip} implies (\ref{eq:WL2}) is immediate; for the converse, it suffices to
consider, for each $y \in \cX$, the sequence
\[
\eps_n = \Bigl( \ell(x,y) - \bigl( f^* - f(x) \bigr) \Bigr)_+ + 1/n\,,
\]
where $( \, \cdot \, )_+$ denotes the nonnegative part.}
to the following property: for all $x \in \cX$ and $\epsilon \geq 0$,
\begin{equation} \label{eq:WL2}
\cB \bigl(x, \, f^*-f(x) + \epsilon \bigr) \, \subset \, \cX_{2 \bigl( f^*-f(x) \bigr) + \epsilon}
\end{equation}
where
\[
\cX_\eps = \bigl\{ x\in \cX : f(x) \geq f^*-\eps \bigr\}
\]
denotes the set of {\em $\epsilon$--optimal arms}. This second property essentially
states that there is no sudden and large drop in the mean-payoff function around
the global maxima (note that this property can be satisfied even for discontinuous functions).

Figure~\ref{fig:WL} presents an illustration of the two properties discussed above.
\begin{figure}[t]
\begin{center}
\includegraphics[scale=0.4]{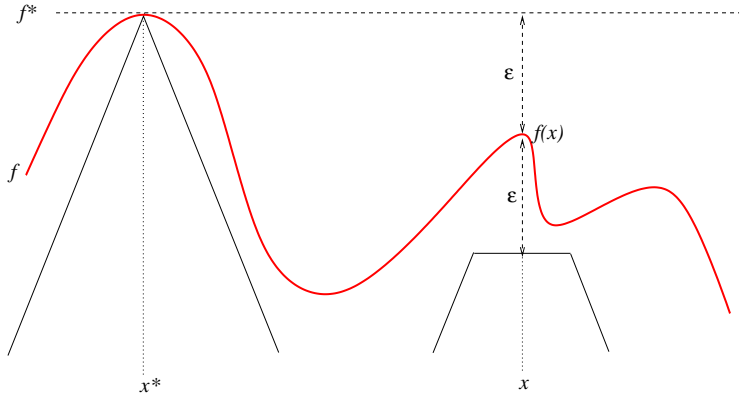}
\caption{Illustration of the property of weak Lipschitzness (on the real line and for
the distance $\ell(x,y) = |x-y|$). Around the optimum $x^*$ the values $f(y)$ should be above $f^* - \ell(x^*,y)$.
Around any $\epsilon$--optimal point $x$ the values
$f(y)$ should be larger than $f^*-2\epsilon$ for $\ell(x,y) \leq \epsilon$
and larger than $f(x)-\ell(x,y)$ elsewhere.}
\label{fig:WL}
\end{center}
\end{figure}

\bigskip
Before stating our main results, we provide a straightforward, though useful
consequence of Assumptions~\ref{ass:shrinking} and~\ref{ass:weaklip}, which
should be seen as an intuitive justification for the third term in (\ref{eq:U}).

For all nodes $(h,i)$, let
\[
f^*_{h,i} = \sup_{x\in \cP_{h,i}} f(x) \qquad
\mbox{and} \qquad
\Delta_{h,i} = f^* - f^*_{h,i}\,.
\]
$\Delta_{h,i}$ is called the {\em suboptimality factor} of node $(h,i)$.
Depending whether it is positive or not,
a node $(h,i)$ is called \emph{suboptimal} ($\Delta_{h,i} > 0$) or \emph{optimal} ($\Delta_{h,i} = 0$).

\begin{lemma}
\label{lem:gooddomains}
Under Assumptions~\ref{ass:shrinking} and~\ref{ass:weaklip},
if the suboptimality factor $\Delta_{h,i}$ of a region $\cP_{h,i}$ is bounded by $c \nu_1 \rho^h$ for some $c \geq 0$,
then all arms in $\cP_{h,i}$ are $\max\{2c,c+1\} \, \nu_1 \rho^h$--optimal, that is,
\[
\cP_{h,i} \subset \cX_{ \max\{2c,c+1\} \, \nu_1 \rho^h }\,.
\]
\end{lemma}

\begin{proof}
For all $\delta > 0$, we denote by $x^*_{h,i}(\delta)$ an element of $\cP_{h,i}$ such that
\[
f \bigl( x^*_{h,i}(\delta) \bigr) \geq \fhi - \delta = f^* - \Delta_{h,i} - \delta\,.
\]
By the weak Lipschitz property (Assumption~\aref{ass:weaklip}), it then follows that for all $y \in \cP_{h,i}$,
\begin{multline}
\nonumber
f^*-f(y) \leq  f^*- f \bigl( x^*_{h,i}(\delta) \bigr) + \max\Bigl\{ f^*- f \bigl( x^*_{h,i}(\delta) \bigr),
\, \ell \bigl( x^*_{h,i}(\delta), \,\, y \bigr) \Bigr\} \\
\leq \Delta_{h,i} + \delta + \max \bigl\{ \Delta_{h,i} + \delta, \,\, \diam \cP_{h,i} \bigr\}\,.
\end{multline}
Letting $\delta \to 0$ and substituting the bounds on the suboptimality and on the diameter
of $\cP_{h,i}$ (Assumption~A1) concludes the proof.
\end{proof}

\subsection{Upper bound for the regret of HOO}
\label{sec:RegrHOO}

\citeauthor{AOS07} \cite[Assumption~2]{AOS07}
observed that the regret of a continuum-armed bandit algorithm
should depend on how fast the volumes of the sets of $\epsilon$--optimal arms shrink
as $\epsilon \ra 0$.
Here, we capture this by defining a new notion, the near-optimality dimension of the mean-payoff function.
The connection between these concepts, as well as with the zooming dimension defined by \citet{KSU08}, will
be further discussed in Section~\ref{sec:disc}.
We start by recalling the definition of packing numbers.

\begin{definition}[Packing number]
The $\eps$--packing number $\cN(\cX,\ell,\epsilon)$
of $\cX$ \wrt\ the dissimilarity $\ell$
is the size of the largest packing of $\cX$ with disjoint $\ell$--open balls of
radius $\epsilon$.
That is, $\cN(\cX,\ell,\epsilon)$ is the largest integer $k$ such that there exists
$k$ disjoint $\ell$--open balls with radius $\epsilon$ contained in $\cX$.
\end{definition}

We now define the $c$--near-optimality dimension, which characterizes the size of the sets $\cX_{c\eps}$ as a function of $\eps$.
It can be seen as some growth rate in $\varepsilon$ of the metric entropy (measured in terms of $\ell$ and with packing numbers
rather than covering numbers) of the set of $c\varepsilon$--optimal arms.

\begin{definition}[Near-optimality dimension]
\label{def:nearopt}
For $c>0$ the {\em $c$--near-optimality dimension} of $f$ \wrt\ $\ell$ equals
$$\max\left\{ 0, \; \limsup_{\epsilon \to 0} \frac{\ln \, \cN \bigl( \cX_{c\eps}, \ell, \, \eps \bigl)}{\ln \bigl(
\epsilon^{-1} \bigr)} \right\}\,.$$
\end{definition}

The following example shows that using a dissimilarity (rather than a metric,
for instance) may sometimes allow for a significant reduction of the near-optimality dimension.
\begin{example} \label{ex:2}
Let $\cX = [0,1]^D$ and let $f: [0,1]^D\ra [0,1]$ be defined by $f(x) = 1 - \|x\|^a$ for some $a \geq 1$ and some norm $\|\cdot\|$ on $\R^D$.
Consider the dissimilarity $\ell$ defined by $\ell(x,y) = \|x-y\|^{a}$.
We shall see in Example~\ref{ex:2:ctd} that $f$ is weakly Lipschitz \wrt\ $\ell$ (in a sense however slightly
weaker than the one given by~\eqref{eq:WL1} and~\eqref{eq:WL2} but sufficiently strong to ensure
a result similar to the one of the main result, Theorem~\ref{th-mainresult} below).
Here we claim that the $c$--near-optimality dimension (for any $c>0$) of $f$ \wrt\ $\ell$ is $0$.
On the other hand, the $c$--near-optimality dimension (for any $c>0$) of $f$
\wrt\ the dissimilarity $\ell'$ defined, for $0 < b < a$, by $\ell'(x,y)=\|x-y\|^b$ is $(1/b-1/a) D>0$.
In particular, when $a > 1$ and $b=1$, the $c$--near-optimality dimension is $(1-1/a)D$.
\end{example}
\begin{quote} {\small
\begin{proof} \textbf{(sketch)}
Fix $c > 0$. The set $\cX_{c \eps}$ is
the $\|\cdot\|$--ball with center $0$
and radius $(c\eps)^{1/a}$, that is, the $\ell$--ball with center $0$
and radius $c\eps$. Its $\eps$--packing number \wrt\
$\ell$ is bounded by a constant depending only on $D$, $c$ and $a$;
hence, the value $0$ for the near-optimality dimension \wrt\ the dissimilarity $\ell$.

In case of $\ell'$, we are interested in the packing number of the
$\|\cdot\|$--ball with center $0$ and radius $(c\eps)^{1/a}$
\wrt\ $\ell'$--balls. The latter is of the order of
\[
\left( \frac{(c\epsilon)^{1/a}}{\epsilon^{1/b}} \right)^D = c^{D/a} \bigl( \eps^{-1} \bigr)^{(1/b-1/a)D}~;
\]
hence, the value $(1/b-1/a) D$ for the near-optimality dimension in the
case of the dissimilarity $\ell'$.

Note that in all these cases the $c$--near-optimality dimension of $f$ is independent of the value of $c$.
\end{proof} }
\end{quote}

We can now state our first main result.
The proof is presented in Section \ref{sec:proof1}.

\begin{theorem}[Regret bound for HOO] \label{th-mainresult}
Consider HOO tuned with parameters such that Assumptions~\ref{ass:shrinking}
and~\ref{ass:weaklip} hold for some dissimilarity $\ell$.
Let $d$ be the $4\nu_1/\nu_2$--near-optimality  dimension
of the mean-payoff function $f$ \wrt\  $\ell$.
Then, for all $d' > d$, there exists a constant $\gamma$ such that for all $n \geq 1$,
\[
\E \bigl[ R_n \bigr] \leq \gamma \, n^{(d'+1)/(d'+2)} \, \bigl( \ln n \bigr)^{1/(d'+2)}.
\]
\end{theorem}
Note that if $d$ is infinite, then the bound is vacuous.
The constant $\gamma$ in the theorem depends on $d'$ and on all other parameters of HOO and of
the assumptions, as well as on the bandit environment $M$. (The value of $\gamma$
is determined in the analysis; it is in particular proportional to $\nu_2^{-d'}$.)
The next section
will exhibit a refined upper bound with a more explicit value of $\gamma$
in terms of all these parameters.

\begin{remark}
The tuning of the parameters of HOO is critical for the assumptions to be
satisfied, thus to achieve a good regret;
given some environment, one should select the parameters of HOO such that
the near-optimality dimension of the mean-payoff function is minimized.
Since the mean-payoff function is unknown to the user, this might be difficult to achieve.
Thus, ideally, these parameters should be selected adaptively based on the observation of
some preliminary sample. For now, the investigation of this possibility is left for future work.
\end{remark}

\subsection{Improving the running time when the time horizon is known}
\label{sec:runningtime}

A deficiency of the basic HOO algorithm is that its computational complexity scales
quadratically with the number of time steps.
In this section we propose a simple modification to HOO that achieves essentially the same regret
as HOO and whose computational complexity scales only log-linearly with the number of time steps.
The needed amount of memory is still linear.
We work out the case when the time horizon, $n_0$, is known in advance.
The case of unknown horizon can be dealt with by resorting to the so-called doubling trick,
see, e.g., \cite[Section~2.3]{CL06}, which consists of periodically restarting the algorithm for regimes of lengths that double
at each such fresh start, so that the $r^{\rm th}$ instance of the algorithm runs for $2^r$ rounds.

\bigskip
We consider two modifications to the algorithm described in Section~\ref{sec:HOO}.
First, the quantities $U_{h,i}(n)$ of (\ref{eq:U}) are redefined
by replacing the factor $\ln n$ by $\ln n_0$, that is, now
\[
U_{h,i}(n) = \wh{\mu}_{h,i}(n) + \sqrt{\frac{2 \ln n_0}{T_{h,i}(n)}} + \nu_1 \rho^h\,.
\]
(This results in a policy which explores the arms with a slightly increased frequency.)
The definition of the $B$--values in terms of the $U_{h,i}(n)$ is unchanged.
A pleasant consequence of the above modification is that the $B$--value of a given node
changes only when this node is part of a path selected by the algorithm.
Thus at each round $n$, only the nodes along
the chosen path need to be updated according to the obtained reward.

However, and this is the reason for the second modification,
in the basic algorithm, a path at round $n$ may be of length linear in $n$
(because the tree could have a depth linear in $n$).
This is why we also truncate the trees $\cT_n$ at a depth $D_{n_0}$ of the order of $\ln n_0$.
More precisely, the algorithm now selects the node $(H_n,I_n)$ to pull at round $n$
by following a path in the tree $\cT_{n-1}$, starting from the root and choosing at each node
the child with the highest $B$--value (with the new definition above using $\ln n_0$),
and stopping either when it encounters a node which has not been expanded before
or a node at depth equal to
\[
D_{n_0} = \left\lceil \frac{(\ln n_0)/2 - \ln (1/\nu_1)}{\ln(1/\rho)} \right\rceil\,.
\]
(It is assumed that $n_0>1/\nu_1^2$ so that $D_{n_0} \geq 1$.)
Note that since no child of a node $(D_{n_0},i)$ located at depth $D_{n_0}$ will ever
be explored, its $B$--value at round $n \leq n_0$ simply equals $U_{D_{n_0},i}(n)$.

We call this modified version of HOO the {\em truncated HOO} algorithm.
The computational complexity of updating all $B$--values at each round $n$ is of the order of
$D_{n_0}$ and thus of the order of $\ln n_0$. The total computational complexity up to round $n_0$
is therefore of the order of $n_0 \ln n_0$, as claimed in the introduction of this section.

As the next theorem indicates
this new procedure enjoys almost the same cumulative regret bound as the basic HOO algorithm.

\begin{theorem}[Upper bound on the regret of truncated HOO]
\label{th:runningtime}
Fix a horizon $n_0$ such that $D_{n_0} \geq 1$.
Then, the regret bound of
 Theorem~\ref{th-mainresult} still holds true at round $n_0$ for truncated HOO
 up to an additional additive $4 \sqrt{n_0}$ factor.
\end{theorem}

\subsection{Local assumptions}
\label{sec:localHOO}

In this section we further relax the weak Lipschitz assumption
and require it only to hold locally around the maxima.
Doing so, we will be able to deal with an even larger class of functions and in fact we will
show that the algorithm studied in this section achieves a $O(\sqrt{n})$ bound on the regret
regret when it is used for functions that are smooth around their maxima
(e.g., equivalent to $\Arrowvert x-x^* \Arrowvert^\alpha$ for some known smoothness degree $\alpha > 0$).

For the sake of simplicity and to derive exact constants
we also state in a more explicit way the assumption on the near-optimality dimension.
We then propose a simple and efficient adaptation of the HOO algorithm suited for this context.

\subsubsection{Modified set of assumptions}

\begin{asss*}
Given the parameters of (the adaption of) HOO,
that is, the real numbers $\nu_1 > 0$ and $\rho \in (0,1)$ and the tree of coverings
$(\cP_{h,i})$, there exists a dissimilarity function $\ell$ such that Assumption~\ref{ass:shrinking}
(for some $\nu_2 > 0$) as well as the following two assumptions hold.
\begin{enumerate}
\addtocounter{assumptionV}{-1} 
\refstepcounter{assumptionV} 
\item[\emph{\theassumptionV}.]  \label{ass:weaklipv}
There exists $\epsilon_0 > 0$ such that
for all optimal subsets $\cA \subset \cX$ (i.e., $\sup_{x \in \cA} f(x) = f^*$) with
diameter $\diam(\cA) \leq \epsilon_0$,
\begin{equation*}
f^*-\inf_{x \in \cA} f(x) \leq  \diam(\cA)\,.
\end{equation*}
Further, there exists $L > 0$ such that for all $x \in \cX_{\epsilon_0}$ and $\epsilon \in [0,\epsilon_0]$,
\[
\cB \bigl(x, \,\, f^*-f(x)+\epsilon \bigr) \, \subset \, \cX_{L \bigl( 2 (f^*-f(x)) + \epsilon \bigr)}\,.
\]
\refstepcounter{assumption}
\item[\emph{\theassumption}.]  \label{ass:a3}
There exist $C>0$ and $d>0$ such that for all $\epsilon \leq \epsilon_0$,
\begin{equation*}
\cN \bigl( \cX_{c\eps}, \, \ell, \, \eps \bigr) \leq C \epsilon^{-d}\,,
\end{equation*}
where $c = 4 L \nu_1/\nu_2$.
\end{enumerate}
\end{asss*}

When $f$ satisfies Assumption~\ref{ass:weaklipv}, we say that
$f$ is $\epsilon_0$--locally $L$--{\em weakly Lipschitz} \wrt\  $\ell$.
Note that this assumption was obtained by weakening the characterizations~(\ref{eq:WL1})
and~(\ref{eq:WL2}) of weak Lipschitzness.

Assumption~\ref{ass:a3} is not a real assumption but merely a reformulation of the definition
of near optimality (with the small added ingredient that the limit can be achieved,
see the second step of the proof of Theorem~\ref{th-mainresult} in Section~\ref{sec:proof1}).

\begin{example}
\label{ex:2:ctd}
We consider again the domain $\cX$ and function $f$ studied in Example~\ref{ex:2} and prove (as
 announced beforehand) that $f$ is $\epsilon_0$--locally $2^{a-1}$--weakly Lipschitz
\wrt\  the dissimilarity $\ell$ defined by $\ell(x,y) = \|x-y\|^a$;
which, in fact, holds for all $\eps_0$.
\end{example}
\begin{quote} {\small
\begin{proof}
Note that $x^* = (0,\ldots,0)$ is such that $f^* = 1 = f(x^*)$.
Therefore, for all $x \in \cX$,
\[
f^* - f(x) = \| x \|^a = \ell(x^*,x)\,,
\]
which yields the first part of Assumption~\ref{ass:weaklipv}.
To prove that the second part is true for $L = 2^{a-1}$ and with no constraint on the
considered $\eps$,
we first note that since $a \geq 1$, it holds by convexity that
$(u+v)^a \leq 2^{a-1}(u^a + v^a)$ for all $u,v \geq 0$.
Now, for all $\epsilon \geq 0$ and $y \in \cB \bigl( x, \, \| x \|^a + \epsilon \bigr)$,
i.e., $y$ such that $\ell(x,y) = \| x-y \|^a \leq \| x \|^a + \epsilon$,
\begin{multline}
\nonumber
f^* - f(y) = \|y\|^a
\leq \bigl( \|x\| + \|x-y\| \bigr)^a
\leq 2^{a-1} \bigl( \|x\|^a + \|x-y\|^a \bigr)
\leq 2^{a-1} \bigl( 2 \| x \|^a + \epsilon \bigr)\,,
\end{multline}
which concludes the proof of the second part of~\ref{ass:weaklipv}.
\end{proof} }
\end{quote}

\subsubsection{Modified HOO algorithm}

We now describe the proposed modifications to the basic HOO algorithm.

We first consider, as a building block, the algorithm called {\em $z$--HOO},
which takes an integer $z$ as an additional parameter to those of HOO.
Algorithm $z$--HOO works as follows: it never plays any node with depth smaller or equal to $z-1$
and starts directly the selection of a new node at depth $z$. To do so, it first picks
the node at depth $z$ with the best $B$--value, chooses a path and then proceeds as
the basic HOO algorithm. Note in particular that the initialization of this algorithm consists
(in the first $2^z$ rounds) in playing once each of the $2^z$ nodes located at depth $z$
in the tree (since by definition a node that has not been played yet has a $B$--value equal
to $+\infty$). We note in passing that when $z = 0$, algorithm $z$--HOO coincides with the basic HOO algorithm.

Algorithm {\em local-HOO} employs the doubling trick in conjunction with consecutive instances of $z$--HOO.
It works as follows.
The integers $r \geq 1$ will index different regimes.
The $r^{\rm th}$ regime starts at round $2^r - 1$ and ends when the next regime starts; it thus
lasts for $2^r$ rounds.
At the beginning of regime $r$, a fresh copy of $z_r$--HOO, where $z_r = \lceil \log_2 r \rceil$,
is initialized and is then used throughout the regime.

Note that each fresh start needs to pull each of the $2^{z_r}$
nodes located at depth $z_r$ at least once (the number of these nodes is $\approx r$).
However, since round $r$ lasts for $2^r$ time steps (which
is exponentially larger than the number of nodes to explore),
the time spent on the initialization of $z_r$--HOO in any regime $r$ is greatly outnumbered
by the time spent in the rest of the regime.

In the rest of this section, we propose first an upper bound on the regret of $z$--HOO (with
exact and explicit constants).
This result will play a key role in proving a bound on the performance of local-HOO.

\subsubsection{Adaptation of the regret bound}

In the following we write $h_0$ for the smallest integer such that
\[
2 \nu_1 \rho^{h_0} < \epsilon_0
\]
and consider the algorithm $z$--HOO, where $z \geq h_0$. In particular, when $z = 0$ is chosen,
the obtained bound is the same as the one of Theorem~\ref{th-mainresult}, except that the constants
are given in analytic forms.

\begin{theorem}[Regret bound for $z$--HOO] \label{th:zHOO}
Consider $z$--HOO tuned with parameters $\nu_1$ and $\rho$ such that Assumptions~\ref{ass:shrinking}, \ref{ass:weaklipv}
and~\ref{ass:a3} hold for some dissimilarity $\ell$ and the values $\nu_2,\,L,\,\epsilon_0,\,C,\,d$.
If, in addition,
$z \geq h_0$ and $n\ge 2$ is large enough so that
\[
z \leq \frac{1}{d+2} \frac{\ln(4 L \nu_1 n) - \ln(\gamma \ln n)}{\ln(1/\rho)}\,,
\]
where
\[
\gamma = \frac{4 \, C L \nu_1 \nu_2^{-d}}{(1/\rho)^{d+1} \, - 1} \left( \frac{16}{\nu_1^2\rho^2} + 9 \right)\,,
\]
then the following bound holds for the expected regret of $z$--HOO:
\[
\E \bigl[ R_n \bigr] \leq
\left( 1 + \frac{1}{\rho^{d+2}} \right) \bigl( 4 L \nu_1 n \bigr)^{(d+1)/(d+2)}
( \gamma \ln n)^{1/(d+2)}
+ \bigl( 2^{z} -1 \bigr) \left(\frac{8\, \ln n}{\nu_1^2 \rho^{2 z}} + 4\right)\,.
\]
\end{theorem}

The proof, which is a modification of the proof to Theorem~\ref{th-mainresult}, can be found in Section~\ref{sec:proofzHOO} of the Appendix.
The main complication arises because the weakened assumptions do not allow one to reason about the smoothness
at an arbitrary scale; this is essentially due to the threshold $\eps_0$
used in the formulation of the assumptions. This is why in the proposed variant of HOO we discard nodes located too close to the root
(at depth smaller than $h_0 - 1$).
Note that in the bound  the second term arises from playing
in regions corresponding to the descendants of ``poor'' nodes located at level $z$.
In particular, this term disappears when $z=0$,
in which case we get a bound on the regret of HOO provided that  $2\nu_1<\epsilon_0$ holds.

\begin{example}
\label{ex:3}
We consider again the setting of Examples~\ref{ex:1ctd}, \ref{ex:2}, and~\ref{ex:2:ctd}.
The domain is $\cX=[0,1]^D$ and
the mean-payoff function $f$ is defined by $f(x) = 1 - \|x\|^2_{\infty}$.
We assume that HOO is run with parameters $\rho=(1/4)^{1/D}$ and $\nu_1 = 4$. We already proved that
Assumptions~\ref{ass:shrinking}, \ref{ass:weaklipv} and~\ref{ass:a3} are satisfied with the dissimilarity
$\ell(x,y) = \|x-y\|^2_{\infty}$, the constants $\nu_2 = 1/4$, $L=2$, $d = 0$,
and\footnote{To compute $C$, one can first note that $4 L \nu_1/\nu_2 = 128$;
the question at hand for Assumption~\ref{ass:a3} to be satisfied
is therefore to upper bound the number of balls of radius $\sqrt{\epsilon}$ (\wrt\  the supremum norm
$\| \, \cdot \, \|_\infty$) that can be packed in a ball of radius $\sqrt{128 \epsilon}$, giving rise to
the bound $C \leq \sqrt{128}^D$.}
$C = 128^{D/2}$, as well as any $\epsilon_0 > 0$ (that is, with $h_0 = 0$).
Thus, resorting to Theorem~\ref{th:zHOO} (applied with $z=0$), we obtain
\[
\gamma = \frac{32 \times 128^{D/2}}{4^{1/D}-1} \bigl( 4^{2/D} + 9 \bigr)
\]
and get
\[
\E \bigl[ R_n \bigr] \leq \bigl( 1 + 4^{2/D} \bigr) \sqrt{32\gamma \,n \ln n}
= \sqrt{\exp \bigl( O(D) \bigr) \, n \ln n}\,.
\]
Under the prescribed assumptions, the rate of convergence is of order $\sqrt{n}$ no matter the ambient dimension $D$.
Although the rate is independent of $D$, the latter impacts the performance through the multiplicative
factor in front of the rate, which is exponential in $D$.
This is, however, not an artifact of our analysis, since it is natural that exploration in a $D$--dimensional
space comes at a cost exponential in $D$.
(The exploration performed by HOO naturally combines an initial global search,
which is bound to be exponential in $D$, and a local optimization, whose regret is of the order of $\sqrt{n}$.)
\end{example}

The following theorem is an almost straightforward consequence of Theorem~\ref{th:zHOO}
(the detailed proof can be found in Section~\ref{sec:thlocalHOO} of the Appendix).
Note that local-HOO does not require the knowledge of the parameter $\epsilon_0$ in \ref{ass:weaklipv}.

\begin{theorem}[Regret bound for local-HOO] \label{th:localHOO}
Consider local-HOO and assume that its parameters are tuned such that
Assumptions~\ref{ass:shrinking}, \ref{ass:weaklipv} and~\ref{ass:a3} hold for some dissimilarity $\ell$. Then the expected regret of local-HOO is bounded (in a distribution-dependent sense) as follows,
\[
\E \bigl[ R_n \bigr] = \tilde{O} \Bigl( n^{(d+1)/(d+2)} \Bigr)\,.
\]
\end{theorem}

\subsection{Minimax optimality in metric spaces} \label{sec:minimax}

In this section we provide two theorems showing the minimax optimality of HOO in
metric spaces. The notion of packing dimension is key.
\begin{definition}[Packing dimension]
The $\ell$--packing dimension of a set $\cX$ (\wrt\  a dissimilarity $\ell$) is defined as
$$\limsup_{\epsilon \to 0} \,\, \frac{\ln \cN(\cX, \ell, \epsilon)}{\ln(\epsilon^{-1})}\,.$$
\end{definition}
For instance, it is easy to see that whenever $\ell$ is a norm,
compact subsets of $\R^D$ with non-empty interiors have a packing dimension of $D$.
We note in passing that the packing dimension provides a bound on the near-optimality dimension that
only depends on $\cX$ and $\ell$ but not on the underlying mean-payoff function.
\newline

Let $\cF_{\cX,\ell}$ be the class of all bandit environments on $\cX$
with a weak Lipschitz mean-payoff function (i.e., satisfying Assumption~\ref{ass:weaklip}).
For the sake of clarity, we now denote, for a bandit strategy $\phi$ and a bandit environment $M$ on $\cX$,
the expectation of the cumulative regret of $\phi$ over $M$ at time $n$
by $\E_{M} \bigl[ R_n(\phi) \bigr]$.

The following theorem provides a uniform upper bound on the regret of HOO over this class of environments.
It is a corollary of Theorem~\ref{th:zHOO}; most of the efforts in the proof consist of showing that the distribution-dependent
constant $\gamma$ in the statement of Theorem~\ref{th:zHOO}
can be upper bounded by a quantity (the $\gamma$ in the statement below) that only depends on $\cX,\,\nu_1,\,\rho,\,\ell,\,\nu_2,\,D'$,
but not on the underlying mean-payoff functions. The proof is provided
in Section~\ref{sec:thminimaxUB} of the Appendix.

\begin{theorem}[Uniform upper bound on the regret of HOO]
\label{th:minimaxUB}
Assume that $\cX$ has a finite $\ell$--packing dimension $D$ and that the parameters of HOO are such
that~\ref{ass:shrinking} is satisfied.
Then, for all $D' > D$ there exists a constant $\gamma$ such that for all $n \geq 1$,
\[
\sup_{M \in \cF_{\cX,\ell}} \E_M \bigl[ R_n( \mbox{\small \rm HOO} ) \bigr]
\leq \gamma \, n^{(D'+1)/(D'+2)} \, \bigl( \ln n \bigr)^{1/(D'+2)}\,.
\]
\end{theorem}

The next result shows that in the case of metric spaces
this upper bound is optimal up to a multiplicative logarithmic factor.
Similar lower bounds appeared in \cite{Kle04} (for $D=1$) and in \cite{KSU08}. We propose
here a weaker statement that suits our needs.
Note that if $\cX$ is a large enough compact subset of $\R^D$
with non-empty interior and the dissimilarity $\ell$ is
some norm of $\R^D$, then the assumption of the following theorem is satisfied.

\begin{theorem}[Uniform lower bound] \label{th:minimaxLB}
Consider a set $\cX$ equipped with a dissimilarity $\ell$ that is a metric.
Assume that there exists some constant $c \in (0,1]$ such that for
all $\epsilon\le 1$, the packing numbers satisfy $\cN(\cX,\ell,\epsilon) \geq c\,\epsilon^{-D} \geq 2$.
Then, there exist two constants $N(c,D)$ and $\gamma(c,D)$ depending only on $c$ and $D$
such that for all bandit strategies $\phi$ and all $n\ge N(c,D)$,
\[
\sup_{M \in \cF_{\cX,\ell}} \,\, \E_M \bigl[ R_n(\varphi) \bigr] \geq  \gamma(c,D) \,\, n^{{(D+1)}/{(D+2)}}\,.
\]
\end{theorem}

The reader interested in the explicit expressions of $N(c,D)$ and $\gamma(c,D)$
is referred to the last lines of the proof of the theorem in the Appendix.

\section{Discussion}
\label{sec:disc}

In this section we would like to shed some light on the results of the previous sections. In particular we generalize the situation of
Example~\ref{ex:3}, discuss the regret that we can obtain, and compare it with what could be obtained by previous works.

\subsection{Examples of regret bounds for functions locally smooth at their maxima} \label{sec:example}
We equip $\cX=[0,1]^D$ with a norm $\|\,\cdot\,\|$. We assume that the mean-payoff function $f$
has a finite number of global maxima and that it is locally
equivalent to the function $\|x-x^*\|^\alpha$ --with degree $\alpha\in [0,\infty)$-- around each such global maximum $x^*$ of $f$;
that is,
\[
f(x^*)-f(x) = \Theta \bigl( \|x-x^*\|^\alpha \bigr) \qquad \mbox{as} \quad  x\rightarrow x^*.
\]
This means that there exist $c_1,c_2,\delta>0$ such that for all $x$ satisfying $\|x-x^*\|\leq \delta$,
\[
c_2\|x-x^*\|^\alpha\leq f(x^*)-f(x)\leq c_1\|x-x^*\|^\alpha\,.
\]
In particular, one can check that Assumption~\ref{ass:weaklipv} is satisfied
for the dissimilarity defined by $\ell_{c,\beta}(x,y) = c \|x-y\|^\beta$,
where $\beta \leq \alpha$ (and $c \geq c_1$ when $\beta=\alpha$).
We further assume that HOO is run with parameters $\nu_1$ and $\rho$ and a tree of dyadic partitions
such that Assumption~\ref{ass:shrinking} is satisfied as well
(see Examples~\ref{ex:1} and~\ref{ex:1ctd} for explicit values of these parameters in the case of the Euclidean or the supremum norms
over the unit cube). The following statements can then be formulated on the expected regret of HOO.
\begin{itemize}
  \item \textbf{Known smoothness:} If we know the true smoothness of $f$ around its maxima, then we set $\beta = \alpha$ and $c \geq c_1$.
  This choice $\ell_{c_1,\alpha}$ of a dissimilarity is such that $f$ is locally weak-Lipschitz with respect to it and
  the near-optimality dimension is $d = 0$ (cf.\ Example~\ref{ex:2}). Theorem~\ref{th:localHOO} thus implies that the expected
  regret of local-HOO is $\tilde O(\sqrt{n})$, i.e., \emph{the rate of the bound is independent of the dimension $D$}.
  \item \textbf{Smoothness underestimated:} Here, we assume that the true smoothness of $f$ around its maxima is unknown and that it is
  underestimated by choosing $\beta < \alpha$ (and some $c$).
  Then $f$ is still locally weak-Lipschitz with respect to the dissimilarity $\ell_{c,\beta}$ and the near-optimality dimension is
  $d = D (1/\beta - 1/\alpha)$, as shown in Example~\ref{ex:2}; the regret of HOO is $\tilde O \bigl(n^{(d+1)/(d+2)} \bigr)$.
  \item \textbf{Smoothness overestimated:} Now, if the true smoothness is overestimated by choosing $\beta > \alpha$ or $\alpha = \beta$ and
  $c <c_1$, then the assumption of weak Lipschitzness is violated and we are unable to provide any guarantee on the behavior of HOO.
  The latter, when used with an overestimated smoothness parameter, may lack exploration and exploit too heavily from the beginning.
  As a consequence, it may get stuck in some local optimum of $f$, missing the global one(s) for a very long time (possibly indefinitely). Such a behavior is illustrated in the example provided in~\cite{CM07} and showing the possible problematic behavior of the closely related algorithm UCT of~\cite{KS06}. UCT is an example of an algorithm overestimating the smoothness of the function; this is because the $B$--values of UCT are defined similarly to the ones of the HOO algorithm but without the third term in the definition (\ref{eq:U}) of the $U$--values. This
  corresponds to an assumed infinite degree of smoothness (that is, to a locally constant mean-payoff function).
\end{itemize}

\subsection{Relation to previous works}

Several works \citep{Agr95b,Kle04,Cop04,AOS07,KSU08} have considered continuum-armed bandits in Euclidean or, more generally,
normed or metric spaces and provided upper and lower bounds on the regret for given classes of environments.
\begin{itemize}
\item \citet{Cop04} derived a $\tilde O(\sqrt{n})$ bound on the regret for compact
and convex subsets of $\R^d$ and mean-payoff functions with a unique minimum and second-order smoothness.
\item \citet{Kle04} considered mean-payoff functions $f$ on the real line that are H{\"o}lder continuous
with degree $0 < \alpha \le 1$. The derived regret bound is $\Theta \bigl( n^{(\alpha+1)/(\alpha+2)} \bigr)$.
\item \citet{AOS07} extended the analysis to classes of functions that are
equivalent to $\Arrowvert x-x^*\Arrowvert^\alpha$ around their maxima $x^*$,
where the allowed smoothness degree is also larger: $\alpha\in [0,\infty)$. They derived the regret bound
\[
\Theta \Bigl( n^{\frac{1+\alpha -\alpha\beta}{1+2\alpha-\alpha\beta}} \Bigr)\,,
\]
where the parameter $\beta$ is such that the Lebesgue measure of $\eps$--optimal arm is $O(\eps^\beta)$.
\item Another setting is the one of \cite{KSU08} and \cite{KSU08ext}, who considered a space $(\cX,\ell)$ equipped with
some dissimilarity $\ell$ and assumed that $f$ is Lipschitz \wrt\ $\ell$ at some
 maximum $x^*$ (when the latter exists and a relaxed condition otherwise),
that is,
\begin{equation} \label{eq:20assumption}
\forall x \in \cX, \qquad f(x^*) - f(x) \leq \ell(x,x^*)\,.
\end{equation}
The obtained regret bound is
$\tilde O \bigl( n^{(d+1)/(d+2)} \bigr)$, where $d$ is the \emph{zooming dimension}. The latter is defined similarly
to our near-optimality dimension with the exceptions that in the definition of zooming dimension
{\em (i)} covering numbers instead of packing numbers are used and {\em (ii)}
 sets of the form $\cX_{\eps}\setminus \cX_{\eps/2}$ are considered instead of the set $\cX_{c\eps}$.
When $(\cX,\ell)$ is a metric space, covering and packing numbers are within a constant factor to each other,
and therefore, one may prove that the zooming and near-optimality dimensions are also equal.
\end{itemize}

For an illustration, consider again the example of Section~\ref{sec:example}.
The result of~\citet{AOS07} shows that for $D=1$, the regret is $\Theta(\sqrt{n})$ (since here $\beta=1/\alpha$,
with the notation above). Our result extends the $\sqrt{n}$ rate of the regret bound to any dimension $D$.

On the other hand the analysis of \citet{KSU08ext} does not apply because in this example
$f(x^*) - f(x)$ is controlled only when $x$ is close in some sense to $x^*$
(i.e., when $\Arrowvert x-x^* \Arrowvert \leq \delta$), while \eqref{eq:20assumption} requires such a control
over the whole set $\cX$. However, note that the local weak-Lipschitz assumption~\ref{ass:weaklipv}
requires an extra condition in the vicinity of $x^*$ compared to
\eqref{eq:20assumption} as it is based on the notion of weak Lipschitzness.
Thus, \ref{ass:weaklipv} and \eqref{eq:20assumption} are in general incomparable (both capture a different phenomenon at the maxima).

\bigskip
We now compare our results to those of~\cite{KSU08} and~\cite{KSU08ext} under Assumption~\ref{ass:weaklip} (which does not cover the example of Section~\ref{sec:example} unless $\delta$ is large). Under this assumption, our algorithms enjoy essentially the same theoretical guarantees as the zooming algorithm of  \cite{KSU08,KSU08ext}. Further, the following hold.
\begin{itemize}
\item Our algorithms do not require the oracle needed by the zooming algorithm.
\item Our truncated HOO algorithm achieves a computational complexity of
 order $O( n \log n)$, whereas the complexity of a naive implementation of the zooming algorithm is
 likely to be much larger.\footnote{The zooming algorithm requires a covering oracle that is able to return a point which is not covered by the set of active strategies, if there exists one. Thus a straightforward implementation of this covering oracle might be computationally expensive in (general) continuous spaces and would require a `global' search over the whole space.}
\item Both truncated HOO and the zooming algorithms use the doubling trick.
The basic HOO algorithm, however, avoids the doubling trick, while meeting the computational complexity of the zooming algorithm.
\end{itemize}
The fact that the doubling trick can be avoided is good news since an algorithm that uses the doubling trick must
start from {\em tabula rasa} time to time,
which results in predictable, yet inevitable, sharp performance drops --a quite unpleasant property.
In particular, for this reason algorithms that rely on the doubling trick are often neglected by practitioners.
In addition, the fact that we avoid the oracle needed by the zooming algorithm is attractive as this oracle might be difficult
to implement for general (non-metric) dissimilarities.

\subsection*{Acknowledgements}
We thank one of the anonymous referee for his valuable comments, which helped us to
provide a fair and detailed comparison of our work to prior contributions.

This work was supported in part by French National Research Agency (ANR, project
EXPLO-RA, ANR-08-COSI-004),
the Alberta Ingenuity Centre of Machine Learning, Alberta Innovates Technology Futures (formerly iCore and AIF), NSERC and the PASCAL2 Network of Excellence under EC grant no. 216886.

\newpage

\appendix
\section{Proofs}

\subsection{Proof of Theorem~\ref{th-mainresult} (main upper bound on
the regret of HOO)}
\label{sec:proof1}

We begin with three lemmas. The proofs of Lemmas~\ref{lem:2} and \ref{lem:3} rely on concentration-of-measure
techniques, while the one of Lemma~\ref{lem:1} follows from a simple case study. Let us fix
some path $(0,1)$, $(1,i^*_1)$, $(2,i^*_2)$, $\ldots\,$ of optimal nodes,
starting from the root. That is, denoting $i^*_0 = 1$, we mean that for all $j \geq 1$, the suboptimality
of $(j,i^*_j)$ equals $\Delta_{j,i^*_j} = 0$ and $(j,i^*_j)$ is a child of $(j-1,i^*_{j-1})$.

\begin{lemma} \label{lem:1}
Let $(h,i)$ be a suboptimal node.
Let $0 \leq k \leq h-1$ be the largest depth such that $(k,i^*_k)$ is on the path from the root $(0,1)$ to $(h,i)$.
Then for all integers $u \geq 0$, we have
\begin{multline}
\nonumber
\E \bigl[ T_{h,i}(n) \bigr] \leq u
+ \sum_{t = u+1}^{n} \P \Bigl\{ \bigl[ U_{s,i^*_s}(t) \leq f^* \ \mbox{\rm for some} \ s \in \{ k+1,\ldots,t-1 \} \bigr] \\
\mbox{\rm or} \ \ \,\, \bigl[ T_{h,i}(t) > u \ \,\, \mbox{\rm and} \,\, \ U_{h,i}(t) > f^* \bigr] \Bigr\}\,.
\end{multline}
\end{lemma}

\begin{proof}
Consider a given round $t \in \{ 1,\ldots,n \}$. If $\hit$, then
this is because the child $(k+1,i')$ of $(k,i^*_k)$ on the path to $(h,i)$ had a better
$B$--value than its brother $(k+1,i^*_{k+1})$. Since by definition, $B$--values can only increase
on a chosen path, this entails that $B_{k+1,i^*_{k+1}} \leq B_{k+1,i'}(t) \leq B_{h,i}(t)$.
This is turns implies, again by definition of the $B$--values,
that $B_{k+1,i^*_{k+1}}(t) \leq U_{h,i}(t)$. Thus,
\[
\bigl\{ \hit \bigr\} \subset \bigl\{ U_{h,i}(t) \geq B_{k+1,i^*_{k+1}}(t) \bigr\}
\subset \bigl\{ U_{h,i}(t) > f^* \bigr\} \cup \bigr\{ B_{k+1,i^*_{k+1}}(t) \leq f^* \bigr\}\,.
\]
But, once again by definition of $B$--values,
\[
\bigr\{ B_{k+1,i^*_{k+1}}(t) \leq f^* \bigr\}
\subset \bigr\{ U_{k+1,i^*_{k+1}}(t) \leq f^* \bigr\} \cup
\bigr\{ B_{k+2,i^*_{k+2}}(t) \leq f^* \bigr\}\,,
\]
and the argument can be iterated. Since up to round $t$ no more than $t$ nodes have been played (including the
suboptimal node $(h,i)$),
we know that $(t,i^*_t)$ has not been played so far and thus has a $B$--value equal to $+\infty$.
(Some of the previous optimal nodes could also have had an infinite $U$--value, if not played so far.)
We thus have proved the inclusion
\beq\label{eq:uhicontainment}
\bigl\{ \hit \bigr\} \subset
\bigl\{ U_{h,i}(t) > f^* \bigr\} \cup \left( \bigr\{ U_{k+1,i^*_{k+1}}(t) \leq f^* \bigr\}
\cup \ldots \cup \bigr\{ U_{t-1,i^*_{t-1}}(t) \leq f^* \bigr\} \right)\,.
\eeq
Now, for any integer $u\ge 0$ it holds that
\beqan
T_{h,i}(n)
&=& \sum_{t=1}^{n} \oneb{ \{ \hit,\,\,T_{h,i}(t)\le u \} }
+ \sum_{t=1}^{n} \oneb{ \{ \hit,\,\,T_{h,i}(t) >  u \} }\\
&\le& u +  \sum_{t=u+1}^{n} \oneb{ \{ \hit,\,\,T_{h,i}(t) >  u \} }\,,
\eeqan
where we used for the inequality the fact that
the quantities $T_{h,i}(t)$ are constant from $t$ to $t+1$,
except when $\hit$, in which case, they increase by 1;
therefore, on the one hand, at most $u$ of the $T_{h,i}(t)$ can be smaller than $u$ and
on the other hand, $T_{h,i}(t) > u$ can only happen if $t > u$.
Using~\eqref{eq:uhicontainment} and then
taking expectations yields the result.
\end{proof}

\begin{lemma} \label{lem:2}
Let Assumptions~\ref{ass:shrinking} and~\ref{ass:weaklip} hold.
Then, for all optimal nodes $(h,i)$ and for all integers $n \geq 1$,
\[
\P \bigl\{ U_{h,i}(n) \leq f^* \bigr\} \leq n^{-3}\,.
\]
\end{lemma}

\begin{proof}
On the event that $(h,i)$ was not played during the first $n$ rounds, one has,
by convention, $U_{h,i}(n) = +\infty$.
In the sequel, we therefore restrict our attention
to the event $\bigl\{ T_{h,i}(n) \geq 1 \bigr\}$.

Lemma~\ref{lem:gooddomains} with $c = 0$ ensures that
$f^* - f(x) \leq \nu_1 \rho^h$ for all arms $x\in \cP_{h,i}$.
Hence,
\[
\sum_{t=1}^n \bigl( f(X_t) + \nu_1 \rho^h - f^* \bigr) \, \oneb{ \{ \hit \} }
\geq 0
\]
and therefore,
\begin{eqnarray*}
\lefteqn{
\P \bigl\{ U_{h,i}(n) \leq f^* \ \ \,\, \mbox{and} \,\, \ \ T_{h,i}(n) \geq 1 \bigr\} } \\
& = &
\P \left\{ \wh{\mu}_{h,i}(n) + \sqrt{\frac{2 \ln n}{T_{h,i}(n)}} + \nu_1 \rho^h \leq f^* \ \ \,\, \mbox{and} \,\, \ \ T_{h,i}(n) \geq 1
\right\} \\
& = & \P \left\{ {T_{h,i}(n)} \, \wh{\mu}_{h,i}(n) + {T_{h,i}(n)}\,\bigl( \nu_1 \rho^h - f^* \bigr)
\leq - \sqrt{2 \, T_{h,i}(n) \ln n} \ \ \,\, \mbox{and} \,\, \ \ T_{h,i}(n) \geq 1 \right\} \\
& = & \P \Biggl\{ \sum_{t=1}^n \bigl( Y_t - f(X_t) \bigr) \oneb{ \{ \hit \} }
+ \sum_{t=1}^n \bigl( f(X_t) + \nu_1 \rho^h - f^* \bigr) \oneb{ \{ \hit \} } \\
& & \qquad \leq - \sqrt{2 \, T_{h,i}(n) \ln n} \ \ \,\, \mbox{and} \,\, \ \ T_{h,i}(n) \geq 1 \Biggr\} \\
& \leq & \P \left\{ \sum_{t=1}^n \bigl( f(X_t) - Y_t \bigr) \oneb{ \{ \hit \} }
\geq \sqrt{2 \, T_{h,i}(n) \ln n} \ \ \,\, \mbox{and} \,\, \ \ T_{h,i}(n) \geq 1 \right\}\,.
\end{eqnarray*}
We take care of the last term with a union bound and the Hoeffding-Azuma inequality for martingale differences.

To do this in a rigorous manner, we need to define a sequence of (random) stopping times when arms in $\cC(h,i)$ were pulled:
\[
T_j = \min \bigl\{ t : \ \ T_{h,i}(t) = j \bigr\}\,, \quad j=1, 2, \hdots\,.
\]
Note that $1\le T_1 < T_2 < \ldots$, hence it holds that $T_j\ge j$.
We denote by $\tilde{X}_j = X_{T_j}$ the $j^{\rm th}$ arm pulled in the region corresponding to
$\cC(h,i)$. Its associated corresponding reward
equals $\tilde{Y}_j = Y_{T_j}$ and
\begin{eqnarray*}
\lefteqn{ \P \left\{ \sum_{t=1}^n \bigl( f(X_t) - Y_t \bigr) \oneb{ \{ \hit \} }
\geq \sqrt{2 \, T_{h,i}(n) \ln n} \ \ \,\, \mbox{and} \,\, \ \ T_{h,i}(n) \geq 1 \right\} } \\
& = &\P \left\{ \sum_{j=1}^{T_{h,i}(n)} \Bigl( f \bigl( \tilde{X}_j \bigr) - \tilde{Y}_j \Bigr)
\geq \sqrt{2 \, T_{h,i}(n) \ln n} \ \ \,\, \mbox{and} \,\, \ \ T_{h,i}(n) \geq 1 \right\} \\
& \leq & \sum_{t=1}^n \ \,\, \P \left\{ \sum_{j=1}^{t} \Bigl( f \bigl( \tilde{X}_j \bigr) - \tilde{Y}_j \Bigr)
\geq \sqrt{2\,t \ln n} \right\}\,,
\end{eqnarray*}
where we used a union bound to get the last inequality.

We claim that
\[
Z_t = \sum_{j=1}^{t} \Bigl( f \bigl( \tilde{X}_j \bigr) - \tilde{Y}_j \Bigr)
\]
is a martingale \wrt\ the filtration $\cG_t = \sigma \bigl( \tilde{X}_1, Z_1, \ldots, \tilde{X}_t, Z_t, \tilde{X}_{t+1} \bigr)$.
This follows, via optional skipping
(see \cite[Chapter VII, adaptation of Theorem 2.3]{Doo53}),
from the facts that
\[
\sum_{t=1}^n \bigl( f(X_t) - Y_t \bigr) \oneb{ \{ \hit \} }
\]
is a martingale \wrt\ the filtration $\cF_t = \sigma(X_1,Y_1,\ldots,X_t,Y_t,X_{t+1})$ and
that the events $\{ T_j = k \}$ are $\cF_{k-1}$--measurable for all $k \geq j$.

Applying the Hoeffding-Azuma inequality for martingale differences (see \cite{Hoe63}), using
the boundedness of the ranges of the induced martingale difference sequence, we then get, for each $t \geq 1$,
\[
\P \left\{ \sum_{j=1}^{t} \Bigl( f \bigl( \tilde{X}_j \bigr) - \tilde{Y}_j \Bigr)
\geq \sqrt{2\,t \ln n} \right\} \leq \exp \! \left( - \frac{2 \left( \sqrt{2\,t \ln n} \right)^2}{t} \right) = n^{-4}\,,
\]
which concludes the proof.
\end{proof}

\begin{lemma} \label{lem:3}
For all integers $t \leq n$,
for all suboptimal nodes $(h,i)$ such that $\Delta_{h,i} > \nu_1 \rho^h$,
and for all integers $u \geq 1$ such that
\[
u \geq \frac{8 \ln n}{(\Delta_{h,i} - \nu_1 \rho^h)^2}\,,
\]
one has
\[
\P \bigl\{ U_{h,i}(t) > f^* \ \,\, \mbox{\rm and} \ \,\, T_{h,i}(t) > u \bigr\} \leq t\,n^{-4}\,.
\]
\end{lemma}

\begin{proof}
The $u$ mentioned in the statement of the lemma are such that
\[
\frac{\Delta_{h,i} - \nu_1 \rho^h}{2} \geq \sqrt{\frac{2 \ln n}{u}}\,,
\qquad \mbox{thus} \qquad
\sqrt{\frac{2 \ln t}{u}} + \nu_1 \rho^h \leq \frac{\Delta_{h,i} + \nu_1 \rho^h}{2}\,.
\]
Therefore,
\begin{eqnarray*}
\lefteqn{\P \bigl\{ U_{h,i}(t) > f^* \ \ \,\, \mbox{and} \,\, \ \ T_{h,i}(t) > u \bigr\}} \\
& = & \P \left\{ \wh{\mu}_{h,i}(t) + \sqrt{\frac{2 \ln t}{T_{h,i}(t)}} + \nu_1 \rho^h
> \fhi + \Delta_{h,i} \ \ \,\, \mbox{and} \,\, \ \ T_{h,i}(t) > u \right\} \\
& \leq & \P \left\{ \wh{\mu}_{h,i}(t) > \fhi + \frac{\Delta_{h,i} - \nu_1 \rho^h}{2} \ \
\,\, \mbox{and} \,\, \ \ T_{h,i}(t) > u \right\} \\
& \leq & \P \left\{ T_{h,i}(t) \left( \wh{\mu}_{h,i}(t) - \fhi \right) >
\frac{\Delta_{h,i} - \nu_1 \rho^h}{2}\,T_{h,i}(t) \ \ \,\, \mbox{and} \,\, \ \ T_{h,i}(t) > u \right\} \\
& = & \P \left\{ \sum_{s=1}^t \bigl( Y_s - f^*_{h,i} \bigr) \oneb{ \{ \his \} }
> \frac{\Delta_{h,i} - \nu_1 \rho^h}{2}\,T_{h,i}(t) \ \ \,\, \mbox{and} \,\, \ \ T_{h,i}(t) > u \right\} \\
& \leq & \P \left\{ \sum_{s=1}^t \bigl( Y_s - f(X_s) \bigr) \oneb{ \{ \his \} }
> \frac{\Delta_{h,i} - \nu_1 \rho^h}{2}\,T_{h,i}(t) \ \ \,\, \mbox{and} \,\, \ \ T_{h,i}(t) > u \right\}.
\end{eqnarray*}
Now it follows from the same arguments as in the proof of Lemma~\ref{lem:2}
(optional skipping, the Hoeffding-Azuma inequality, and a union bound) that
\begin{eqnarray*}
\lefteqn{\P \left\{ \sum_{s=1}^t \bigl( Y_s - f(X_s) \bigr) \, \oneb{ \{ \his \} }
> \frac{\Delta_{h,i} - \nu_1 \rho^h}{2}\,T_{h,i}(t) \ \ \,\, \mbox{and} \,\, \ \ T_{h,i}(t) > u \right\}} \\
& \leq & \sum_{s'=u+1}^t \exp \left(- \frac{2}{s'} \, \left( \frac{(\Delta_{h,i} - \nu_1 \rho^h)}{2}\,s' \right)^{\!\!2} \,\,
\right) \leq \sum_{s'=u+1}^t \exp \left(- \frac{1}{2} \, s' \, (\Delta_{h,i} - \nu_1 \rho^h)^2\, \right) \\
& \leq & t \, \exp \left(- \frac{1}{2} \, u \, \bigl( \Delta_{h,i} - \nu_1 \rho^h \bigr)^2 \right) \leq t \,n^{-4}\,,
\end{eqnarray*}
where we used the stated bound on $u$ to obtain the last inequality.
\end{proof}

Combining the results of Lemmas~\ref{lem:1}, \ref{lem:2}, and~\ref{lem:3}
leads to the following key result bounding the expected number of visits
to descendants of a ``poor'' node.

\begin{lemma}\label{lem:4}
Under Assumptions~\ref{ass:shrinking} and~\ref{ass:weaklip},
for all suboptimal nodes $(h,i)$ with $\Delta_{h,i} > \nu_1 \rho^h$,
we have, for all $n \geq 1$,
\[
\E[T_{h,i}(n)] \leq \frac{8\, \ln n}{(\Delta_{h,i} - \nu_1 \rho^h)^2} + 4\,.
\]
\end{lemma}

\begin{proof}
We take $u$ as the upper integer part of $(8\,\ln n)/(\Delta_{h,i} - \nu_1 \rho^h)^2$
and use union bounds to get from Lemma~\ref{lem:1} the bound
\begin{multline}
\nonumber
\E \bigl[ T_{h,i}(n) \bigr] \leq \frac{8\, \ln n}{(\Delta_{h,i} - \nu_1 \rho^h)^2} + 1 \\
+ \sum_{t = u+1}^{n} \left( \P \bigl\{ T_{h,i}(t) > u \ \,\, \mbox{\rm and} \,\, \ U_{h,i}(t) > f^* \bigr\}
+ \sum_{s=1}^{t-1} \P \bigl\{ U_{s,i^*_s}(t) \leq f^* \bigr\} \right)\,.
\end{multline}
Lemmas~\ref{lem:2} and~\ref{lem:3} further bound the quantity of interest as
\[
\E \bigl[ T_{h,i}(n) \bigr] \leq \frac{8\, \ln n}{(\Delta_{h,i} - \nu_1 \rho^h)^2} + 1
+ \sum_{t = u+1}^{n} \left( t\,n^{-4} + \sum_{s=1}^{t-1} t^{-3} \right)
\]
and we now use the crude upper bounds
\[
1 + \sum_{t = u+1}^{n} \left( t\,n^{-4} + \sum_{s=1}^{t-1} t^{-3} \right)
\leq 1 + \sum_{t = 1}^{n} \bigl( n^{-3} + t^{-2} \bigr) \leq 2 + \pi^2/6 \leq 4
\]
to get the proposed statement.
\end{proof}

\begin{proof} \textbf{(of Theorem~\ref{th-mainresult})}
First, let us fix $d'>d$.
The statement will be proven in four steps.

\textbf{First step.} For all $h = 0,1,2,\ldots$,
denote by $\cI_h$ the set of those nodes at depth $h$ that are $2\nu_1 \rho^h$--optimal, i.e., the nodes $(h,i)$ such that
$\fhi \geq f^* - 2\nu_1 \rho^h$. (Of course, $\cI_0 = \{ (0,1) \}$.)
Then, let $\cI$ be the union of these sets when $h$ varies.
Further, let $\cJ$ be the set of nodes that are not in $\cI$ but whose parent is in $\cI$.
Finally, for $h = 1,2,\ldots$
we denote by $\cJ_h$ the nodes in $\cJ$ that are located at depth $h$ in the tree
(i.e., whose parent is in $\cI_{h-1}$).

Lemma~\ref{lem:4} bounds in particular the expected
number of times each node $(h,i) \in \cJ_h$ is visited.
Since for these nodes $\Delta_{h,i} > 2\nu_1 \rho^h$, we get
\[
\E \bigl[ T_{h,i}(n) \bigr] \leq \frac{8 \, \ln n}{\nu_1^2\rho^{2h}} + 4\,.
\]

\textbf{Second step.}
We bound the cardinality $|\cI_h|$ of $\cI_h$.
We start with the case $h \geq 1$.
By definition, when $(h,i) \in \cI_h$, one has $\Delta_{h,i} \leq 2\nu_1\rho^h$, so that
by Lemma~\ref{lem:gooddomains} the inclusion $\cP_{h,i} \subset \cX_{4\nu_1\rho^h}$ holds.
Since by Assumption~\ref{ass:shrinking}, the sets $\cP_{h,i}$ contain disjoint
balls of radius $\nu_2 \rho^h$,
we have that
\[
|\cI_h| \leq
\cN \bigl( \cup_{(h,i)\in \cI_h} \cP_{h,i}, \, \ell, \, \nu_2\rho^h \bigr) \leq
\cN \bigl( \cX_{4\nu_1\rho^h}, \, \ell, \, \nu_2\rho^h \bigr) =
\cN \bigl( \cX_{(4\nu_1/\nu_2)\,\nu_2\rho^h}, \, \ell, \, \nu_2\rho^h \bigr)\,.
\]
We prove below that
there exists a constant $C$ such that
for all $\eps \leq \nu_2$,
\begin{equation}
\label{eq:claimC}
\cN \bigl( \cX_{(4\nu_1/\nu_2)\,\epsilon}, \, \ell, \, \epsilon \bigr)
\leq C \, \epsilon^{-d'}\,.
\end{equation}
Thus we obtain the bound $|\cI_h| \leq C\,\bigl( \nu_2 \rho^h \bigr)^{-d'}$ for all $h \geq 1$.
We note that the obtained bound $|\cI_h| \leq C\,\bigl( \nu_2 \rho^h \bigr)^{-d'}$
is still valid for $h = 0$, since $|\cI_0| = 1$.

It only remains to prove~(\ref{eq:claimC}).
Since $d' > d$, where $d$ is the near-optimality of $f$, we have, by definition, that
\[
\limsup_{\eps \to 0} \frac{\ln \cN \bigl( \cX_{(4\nu_1/\nu_2)\,\epsilon}, \, \ell, \, \epsilon \bigr)}{
\ln \bigl( \eps^{-1} \bigr)} \leq d\,,
\]
and thus, there exists $\eps_{d'} > 0$ such that for all $\eps \leq \eps_{d'}$,
\[
\frac{\ln \cN \bigl( \cX_{(4\nu_1/\nu_2)\,\epsilon}, \, \ell, \, \epsilon \bigr)}{
\ln \bigl( \eps^{-1} \bigr)} \leq d'\,,
\]
which in turn implies that for all $\eps \leq \eps_{d'}$,
\[
\cN \bigl( \cX_{(4\nu_1/\nu_2)\,\epsilon}, \, \ell, \, \epsilon \bigr)
\leq \epsilon^{-d'}\,.
\]
The result is proved with $C = 1$ if $\eps_{d'} \geq \nu_2$. Now, consider
the case $\eps_{d'} < \nu_2$.
Given the definition of packing numbers, it is straightforward that
for all $\eps \in \bigl[ \eps_{d'}, \, \nu_2 \bigr]$,
\[
\cN \bigl( \cX_{(4\nu_1/\nu_2)\,\epsilon}, \, \ell, \, \epsilon \bigr)
\leq u_{d'} \stackrel{\mbox{\small{def}}}{=} \cN \bigl( \cX, \, \ell, \, \epsilon_{d'} \bigr)~;
\]
therefore,
for all $\eps \in \bigl[ \eps_{d'}, \, \nu_2 \bigr]$,
\[
\cN \bigl( \cX_{(4\nu_1/\nu_2)\,\epsilon}, \, \ell, \, \epsilon \bigr)
\leq u_{d'} \, \frac{\nu_2^{d'}}{\eps^{d'}} = C \eps^{-d'}
\]
for the choice $C = \max \bigl\{ 1, \,\, u_{d'} \,\nu_2^{d'} \bigr\}$.
Because we take the maximum with 1, the stated inequality also
holds for $\eps \leq \eps^{-d'}$, which concludes the proof of~(\ref{eq:claimC}).

\textbf{Third step.} Let $H \geq 1$ be an integer to be chosen later.
We partition the nodes of the infinite tree $\cT$ into three subsets,
$\cT = \cT^1 \cup \cT^2 \cup \cT^3$, as follows.
Let the set $\cT^1$ contain the descendants of the nodes in $\cI_H$
(by convention, a node is considered its own descendant, hence the nodes of $\cI_H$ are included in $\cT^1$);
let $\cT^2 = \cup_{0\le h<H} \, \cI_h$; and let
$\cT^3$ contain the descendants of the nodes in $\cup_{1 \leq h\le H} \, \cJ_h$.
Thus, $\cT^1$ and $\cT^3$ are potentially infinite, while $\cT^2$ is finite.

We recall that we denote by $(H_t,I_t)$ the node that was chosen by HOO in round $t$.
From the definition of the algorithm,
each node is played at most once, thus no two such random variables are equal when $t$ varies.
We decompose the regret according to which of the sets $\cT^j$ the nodes $(H_t,I_t)$ belong to:
\begin{multline}
\nonumber
\E\bigl[ R_n \bigr] = \E \! \left[ \sum_{t=1}^n (f^* - f(X_t)) \right] = \E \bigl[ R_{n,1} \bigr]
 + \E \bigl[ R_{n,2} \bigr] + \E \bigl[ R_{n,3} \bigr]\,, \\ \mbox{where}
\qquad R_{n,i}  = \sum_{t=1}^n \bigl( f^* - f(X_t) \bigr) \oneb{\{ (H_t,I_t) \in \cT^i \} }\,,
\qquad \mbox{for } i=1,2,3.
\end{multline}
The contribution from $\cT^1$ is easy to bound.
By definition any node in $\cI_H$ is $2\nu_1\rho^H$--optimal.
Hence, by Lemma~\ref{lem:gooddomains}, the corresponding domain is included in $\cX_{4\nu_1\rho^H}$.
By definition of a tree of coverings, the domains of the descendants of these nodes are
still included in $\cX_{4\nu_1\rho^H}$.
Therefore,
\[
\E \bigl[ R_{n,1} \bigr] \leq 4 \nu_1 \rho^H \, n\,.
\]

For $h \geq 0$, consider a node $(h,i) \in \cT^2$.
It belongs to $\cI_{h}$ and is therefore $2\nu_1 \rho^h$--optimal.
By Lemma~\ref{lem:gooddomains}, the corresponding domain is included in $\cX_{4\nu_1 \rho^h}$.
By the result of the second step of this proof and using that each node is played at most once, one gets
\[
\E \bigl[ R_{n,2} \bigr] \leq \sum_{h=0}^{H-1} 4\nu_1 \rho^h \, |\cI_{h}| \leq
4 C \nu_1\nu_2^{-d'}\,\sum_{h=0}^{H-1} \rho^{h(1-d')}\,.
\]

We finish by bounding the contribution from $\cT^3$.
We first remark that since the parent of any element $(h,i) \in \cJ_h$ is in $\cI_{h-1}$,
by Lemma~\ref{lem:gooddomains} again, we have that $\cP_{h,i} \subset
\cX_{4\nu_1\rho^{h-1}}$.
We now use the first step of this proof to get
\[
\E \bigl[ R_{n,3} \bigr] \leq \sum_{h=1}^H 4\nu_1\rho^{h-1}\,\sum_{i\,:\,(h,i) \in \cJ_h}
\E \bigl[ T_{h,i}(n) \bigr] \leq \sum_{h=1}^H 4\nu_1\rho^{h-1}\,|\cJ_h|\,\left(
\frac{8 \, \ln n}{\nu_1^2\rho^{2h}} + 4 \right)\,.
\]
Now, it follows from the fact that the parent of $\cJ_h$ is in $\cI_{h-1}$ that
$|\cJ_h| \leq 2 |\cI_{h-1}|$ when $h \geq 1$. Substituting this and the bound on $|\cI_{h-1}|$ obtained in
the second step of this proof, we get
\begin{eqnarray*}
\E \bigl[ R_{n,3} \bigr]
& \leq & \sum_{h=1}^H 4\nu_1\rho^{h-1}\,\left( 2 C\,\bigl( \nu_2 \rho^{h-1} \bigr)^{-d'} \right)\,\left(
\frac{8 \, \ln n}{\nu_1^2\rho^{2h}} + 4 \right) \\
& \leq & 8 C \nu_1 \nu_2^{-d'} \, \sum_{h=1}^H \rho^{h(1-d')+d'-1}\,\left(
\frac{8 \, \ln n}{\nu_1^2\rho^{2h}} + 4 \right)\,.
\end{eqnarray*}

\textbf{Fourth step.}
Putting the obtained bounds together, we get
\begin{eqnarray}
\E \bigl[ R_n \bigr] & \leq &
4 \nu_1 \rho^H \, n +
4 C \nu_1 \nu_2^{-d'}\,\sum_{h=0}^{H-1} \rho^{h(1-d')} +
8 C \nu_1 \nu_2^{-d'}\,\sum_{h=1}^H \rho^{h(1-d')+d'-1}\,\left(
\frac{8 \, \ln n}{\nu_1^2\rho^{2h}} + 4 \right)
\nonumber \\
& = & O \! \left( n \rho^H
+ (\ln n) \sum_{h=1}^H \rho^{-h(1+d')} \right)
= O \Bigl( n \rho^H + \rho^{-H(1+d')}\ln n \Bigr)
\end{eqnarray}
(recall that $\rho < 1$). Note that all constants hidden in the $O$ symbol
only depend on $\nu_1$, $\nu_2$, $\rho$ and $d'$.

Now, by choosing $H$ such that
$\rho^{-H(d'+2)}$ is of the order of $n/\ln n$, that is,
$\rho^H$ is of the order of $(n/\ln n)^{-1/(d'+2)}$,
we get the desired result, namely,
\[
\E \bigl[ R_n \bigr] = O \Bigl( n^{(d'+1)/(d'+2)}\, (\ln n)^{1/(d'+2)} \Bigr)\,.
\]
\end{proof}

\subsection{Proof of Theorem~\ref{th:runningtime} (regret bound for truncated HOO)}

The proof follows from an adaptation of the proof of Theorem~\ref{th-mainresult} and of its associated lemmas;
for the sake of clarity and precision, we explicitly state the adaptations of the latter. \\

\textbf{Adaptations of the lemmas.}
Remember that $D_{n_0}$ denotes the maximum depth of the tree, given horizon $n_0$.
The adaptation of Lemma~\ref{lem:1} is done as follows.
Let $(h,i)$ be a suboptimal node with $h \leq D_{n_0}$ and let
$0 \leq k \leq h-1$ be the largest depth such that $(k,i^*_k)$ is on the path from the root $(0,1)$ to $(h,i)$.
Then,
for all integers $u \geq 0$, one has
\begin{multline}
\nonumber
\E \bigl[ T_{h,i}(n_0) \bigr] \leq u
+ \sum_{t = u+1}^{n_0} \P \Bigl\{ \bigl[ U_{s,i^*_s}(t) \leq f^* \ \mbox{\rm for some $s$ with} \ k+1 \leq s \leq \min \{ D_{n_0},
n_0 \} \bigr] \\
\mbox{\rm or} \ \ \,\, \bigl[ T_{h,i}(t) > u \ \,\, \mbox{\rm and} \,\, \ U_{h,i}(t) > f^* \bigr] \Bigr\}\,.
\end{multline}

As for Lemma~\ref{lem:2}, its straightforward adaptation states that under
Assumptions~\ref{ass:shrinking} and~\ref{ass:weaklip},
for all optimal nodes $(h,i)$ with $h \leq D_{n_0}$ and for all integers $1\leq t \leq n_0$,
\[
\P \bigl\{ U_{h,i}(t) \leq f^* \bigr\} \leq t \, (n_0)^{-4} \leq (n_0)^3\,.
\]

Similarly, the same changes yield from Lemma~\ref{lem:3} the following result for truncated HOO.
For all integers $t \leq n_0$,
for all suboptimal nodes $(h,i)$ such that
$h \leq D_{n_0}$ and $\Delta_{h,i} > \nu_1 \rho^h$,
and for all integers $u \geq 1$ such that
\[
u \geq \frac{8 \ln n_0}{(\Delta_{h,i} - \nu_1 \rho^h)^2}\,,
\]
one has
\[
\P \bigl\{ U_{h,i}(t) > f^* \ \,\, \mbox{\rm and} \ \,\, T_{h,i}(t) > u \bigr\} \leq t\,(n_0)^{-4}\,.
\]

Combining these three results (using the same methodology as in the proof of Lemma~\ref{lem:4}) shows
that under Assumptions~\ref{ass:shrinking} and~\ref{ass:weaklip},
for all suboptimal nodes $(h,i)$ such that $h \leq D_{n_0}$ and $\Delta_{h,i} > \nu_1 \rho^h$,
one has
\begin{eqnarray*}
\E[T_{h,i}(n_0)] & \leq & \frac{8\, \ln n_0}{(\Delta_{h,i} - \nu_1 \rho^h)^2} + 1 + \sum_{t=u+1}^{n_0} \left( t\,(n_0)^4
+ \sum_{s=1}^{\min \{ D_{n_0}, n_0 \} } (n_0)^{-3} \right) \\
& \leq & \frac{8\, \ln n_0}{(\Delta_{h,i} - \nu_1 \rho^h)^2} + 3\,.
\end{eqnarray*}
(We thus even improve slightly the bound of Lemma~\ref{lem:4}.) \\

\textbf{Adaptation of the proof of Theorem~\ref{th-mainresult}.}
The main change here comes from the fact that trees are cut at the depth $D_{n_0}$.
As a consequence, the sets $\mathcal{I}_h$, $\mathcal{I}$, $\mathcal{J}$, and $\mathcal{J}_h$ are defined
only by referring to nodes of depth smaller than $D_{n_0}$. All steps
of the proof can then be repeated, except the third step;
there, while the bounds on the regret resulting from nodes of $\cT^1$ and $\cT^3$
go through without any changes (as these sets were constructed by considering all descendants
of some base nodes), the bound on the regret $R_{n,2}$ associated with the nodes $\cT^2$
calls for a modified proof since at this stage we used the property that each node is played at most once.
But this is not true anymore for nodes $(h,i)$ located at depth $D_{n_0}$, which
can be played several times. Therefore the proof is modified as follows.

Consider a node at depth $h = D_{n_0}$. Then, by definition of $D_{n_0}$,
\[
h \geq D_{n_0} = \frac{(\ln n_0)/2 - \ln (1/\nu_1)}{\ln(1/\rho)}\,, \qquad
\mbox{that is}, \qquad
\nu_1 \, \rho^h \leq \frac{1}{\sqrt{n_0}}\,.
\]
Since the considered nodes are $2\nu_1\rho^{D_{n_0}}$--optimal, the corresponding domains
are $4\nu_1\rho^{D_{n_0}}$--optimal by Lemma~\ref{lem:gooddomains}, thus
also $4/\sqrt{n_0}$--optimal. The instantaneous regret incurred when playing any of these
nodes is therefore bounded by $4/\sqrt{n_0}$; and the associated
cumulative regret (over $n_0$ rounds) can be bounded by $4\sqrt{n_0}$.
In conclusion, with the notations of~Theorem~\ref{th-mainresult},
we get the new bound
\[
\E \bigl[ R_{n,2} \bigr] \leq \sum_{h=0}^{H-1} 4\nu_1 \rho^h \, |\cI_{h}| + 4\sqrt{n_0}\leq
4\sqrt{n_0} + 4 C \nu_1\nu_2^{-d'}\,\sum_{h=0}^{H-1} \rho^{h(1-d')}\,.
\]
The rest of the proof goes through and only this additional additive factor of $4\sqrt{n_0}$
is suffered in the final regret bound. (The additional factor can be included in the $O$ notation.)

\subsection{Proof of Theorem~\ref{th:zHOO} (regret bound for $z$--HOO)}
\label{sec:proofzHOO}

We start with the following equivalent of Lemma~\ref{lem:gooddomains} in
this new local context.
Remember that $h_0$ is the smallest integer such that
\[
2 \nu_1 \rho^{h_0} < \epsilon_0\,.
\]
\begin{lemma}
\label{lem:gooddomains-local}
Under Assumptions~\ref{ass:shrinking} and~\ref{ass:weaklipv},
for all $h \geq h_0$,
if the suboptimality factor $\Delta_{h,i}$ of a region $\cP_{h,i}$ is bounded by $c \nu_1 \rho^h$ for some $c \in [0,2]$,
then all arms in $\cP_{h,i}$ are $L \max \{ 2c, \, c+1 \, \} \, \nu_1 \rho^h$--optimal, that is,
\[
\cP_{h,i} \subset \cX_{L \max \{ 2c, \, c+1 \} \nu_1 \rho^h}\,.
\]
When $c = 0$, i.e., the node $(h,i)$ is optimal, the bound improves to
\[
\cP_{h,i} \subset \cX_{\nu_1 \rho^h}\,.
\]
\end{lemma}

\begin{proof}
We first deal with the general case of $c \in [0,2]$.
By the hypothesis on the suboptimality of $\cP_{h,i}$, for all $\delta > 0$, there exists
an element $x \in \cX_{c \nu_1 \rho^h + \delta} \, \cap \, \cP_{h,i}$. If $\delta$ is small
enough, e.g., $\delta \in \bigl( 0, \, \eps_0 - 2 \nu_1 \rho^{h_0} \bigr]$, then
this element satisfies $x \in \cX_{\eps_0}$.
Let $y \in \cP_{h,i}$. By Assumption~\ref{ass:shrinking}, $\ell(x,y) \leq \diam(\cP_{h,i}) \leq \nu_1 \rho^h$,
which entails, by denoting $\eps = \max\bigl\{0, \nu_1 \rho^h - (f^*-f(x)) \bigr\}$,
\[
\ell(x,y) \leq \nu_1 \rho^h \leq f^*-f(x) + \eps\,, \qquad \mbox{that is}, \qquad
y \in \cB \bigl(x, \, f^*-f(x) + \epsilon \bigr)\,.
\]
Since $x \in \cX_{\eps_0}$ and $\eps \leq \nu_1 \rho^h \leq \nu_1 \rho^{h_0} < \eps_0$,
the second part of Assumption~\ref{ass:weaklipv} then yields
\[
y \in \cB \bigl(x, \, f^*-f(x) + \epsilon \bigr) \,\, \subset \cX_{L \bigl(2 (f^*-f(x)) + \epsilon\bigr)}\,.
\]
It follows from the definition of $\eps$ that
$f^* - f(x) + \epsilon = \max \bigl\{ f^* - f(x), \, \nu_1 \rho^h \bigr\}$,
and this implies
\[
y \in \cB \bigl(x, \, f^*-f(x) + \epsilon \bigr) \,\, \subset \cX_{L \bigl( f^*-f(x)  + \max \{ f^* - f(x), \, \nu_1 \rho^h \}
\bigr)}\,.
\]
But $x \in \cX_{c \nu_1 \rho^h + \delta}$, i.e., $f^* - f(x) \leq c \nu_1 \rho^h + \delta$, we thus have proved
\[
y \in \cX_{L \bigl( \max \{ 2c, \, c+1 \} \nu_1 \rho^h + 2\delta \bigr)}\,.
\]
In conclusion, $\cP_{h,i} \subset \cX_{L \max \{ 2c, \, c+1 \} \nu_1 \rho^h +  2 L \delta}$
for all sufficiently small $\delta > 0$. Letting $\delta \to 0$ concludes the proof.

In the case of $c = 0$, we resort to the first part of Assumption~\ref{ass:weaklipv},
which can be applied since $\diam(\cP_{h,i}) \leq \nu_1 \rho^h \leq \eps_0$ as already noted above,
and can exactly be restated as indicating that for all $y \in \cP_{h,i}$,
\[
f^* - f(y) \leq \diam(\cP_{h,i}) \leq \nu_1 \rho^h~;
\]
that is, $\cP_{h,i} \subset \cX_{\nu_1 \rho^h}$.
\end{proof}

We now provide an adaptation of Lemma~\ref{lem:4} (actually based on adaptations of Lemmas~\ref{lem:1}
and~\ref{lem:2}), providing the same bound under local conditions that relax the
assumptions of Lemma~\ref{lem:4} to some extent.

\begin{lemma}
\label{lem:4-local}
Consider a depth $z \geq h_0$. Under Assumptions~\ref{ass:shrinking} and~\ref{ass:weaklip}',
the algorithm $z$--HOO satisfies that
for all $n \geq 1$ and all suboptimal nodes $(h,i)$ with $\Delta_{h,i} > \nu_1 \rho^h$ and $h \geq z$,
\[
\E \bigl[ T_{h,i}(n) \bigr] \leq \frac{8\, \ln n}{(\Delta_{h,i} - \nu_1 \rho^h)^2} + 4\,.
\]
\end{lemma}

\begin{proof}
We consider some path $(z,i^*_z), \, (z+1,i^*_{z+1}), \, \ldots$ of optimal nodes,
starting at depth $z$.
We distinguish two cases, depending on whether there exists $z \leq k' \leq h-1$ such that $(h,i) \in \cC(k',i^*_{k'})$
or not.

In the first case, we denote $k'$ the largest such $k$.
The argument of Lemma~\ref{lem:1} can be used without any change and
shows that for all integers $u \geq 0$,
\begin{multline}
\nonumber
\E \bigl[ T_{h,i}(n) \bigr] \leq u
+ \sum_{t = u+1}^{n} \P \Bigl\{ \bigl[ U_{s,i^*_s}(t) \leq f^* \ \mbox{\rm for some} \ s \in \{ k+1,\ldots,t-1 \} \bigr] \\
\mbox{\rm or} \ \ \,\, \bigl[ T_{h,i}(t) > u \ \,\, \mbox{\rm and} \,\, \ U_{h,i}(t) > f^* \bigr] \Bigr\}\,.
\end{multline}

In the second case, we denote by $(z,i_h)$ the ancestor of $(h,i)$ located at depth $z$.
By definition of $z$--HOO, $(H_t,I_t) \in \cC(h,i)$ at some round $t \geq 1$ only if $B_{z,i^*_z}(t) \leq B_{z,i_h}(t)$
and since $B$--values can only increase on a chosen path,
$(H_t,I_t) \in \cC(h,i)$ can only happen if $B_{z,i^*_z}(t) \leq B_{h,i}(t)$. Repeating again the argument of
Lemma~\ref{lem:1}, we get that for all integers $u \geq 0$,
\begin{multline}
\nonumber
\E \bigl[ T_{h,i}(n) \bigr] \leq u
+ \sum_{t = u+1}^{n} \P \Bigl\{ \bigl[ U_{s,i^*_s}(t) \leq f^* \ \mbox{\rm for some} \ s \in \{ z,\ldots,t-1 \} \bigr] \\
\mbox{\rm or} \ \ \,\, \bigl[ T_{h,i}(t) > u \ \,\, \mbox{\rm and} \,\, \ U_{h,i}(t) > f^* \bigr] \Bigr\}\,.
\end{multline}

Now, notice that Lemma~\ref{lem:3} is valid without any assumption.
On the other hand, with the modified assumptions, Lemma~\ref{lem:2} is still true but only for optimal nodes $(h,i)$ with $h \geq h_0$.
Indeed, the only point in its proof where the assumptions were used was in the fourth line, when applying
Lemma~\ref{lem:gooddomains}; here, Lemma~\ref{lem:gooddomains-local} with $c = 0$ provides the needed guarantee.

The proof is concluded with the same computations as in the proof of Lemma~\ref{lem:4}.
\end{proof}

\begin{proof} \textbf{(of Theorem~\ref{th:zHOO})}
We follow the four steps in the
proof of Theorem~\ref{th-mainresult} with some slight adjustments.
In particular, for $h \geq z$, we use the sets of nodes $\cI_h$ and $\cJ_h$ defined therein.

\textbf{First step.}
Lemma~\ref{lem:4-local} bounds the expected
number of times each node $(h,i) \in \cJ_h$ is visited.
Since for these nodes $\Delta_{h,i} > 2\nu_1 \rho^h$, we get
\[
\E \bigl[ T_{h,i}(n) \bigr] \leq \frac{8 \, \ln n}{\nu_1^2\rho^{2h}} + 4\,.
\]

\textbf{Second step.}
We bound here the cardinality $|\cI_h|$.
By Lemma~\ref{lem:gooddomains-local} with $c = 2$,
when $(h,i) \in \cI_h$ and $h \geq z$, one has $\cP_{h,i} \subset \cX_{4 L \nu_1 \rho^h}$.

Now, by Assumption~\ref{ass:shrinking} and by using the same
argument as in the second step of the proof of Theorem \ref{th-mainresult},
\[
|\cI_h| \leq \cN \bigl( \cX_{(4L\nu_1/\nu_2)\,\nu_2\rho^h}, \, \ell, \, \nu_2\rho^h \bigr)\,.
\]
Assumption~\ref{ass:a3} can be applied since $\nu_2\rho^h \leq 2 \nu_1 \rho^h \leq 2 \nu_1 \rho^{h_0} \leq \epsilon_0$
and yields the inequality $|\cI_h| \leq C \bigl( \nu_2 \rho^h \bigr)^{-d}$.

\textbf{Third step.}
We consider some integer $H \geq z$ to be defined by the analysis in the fourth step.
We define a partition of the nodes located at a depth equal to or larger than $z$;
more precisely,
\begin{itemize}
\item $\cT^1$ contains the nodes of $\cI_H$ and their descendants,
\item $\cT^2 = \displaystyle{\bigcup_{z \leq h \leq H-1} \cI_h}$,
\item $\cT^3$ contains the nodes $\displaystyle{\bigcup_{z+1 \leq h \leq H} \cJ_h}$ and their descendants,
\item $\cT^4$ is formed by the nodes $(z,i)$ located at depth $z$ not belonging to $\cI_{z}$, i.e.,
such that $\Delta_{z,i} > 2 \nu_1 \rho^z$, and their descendants.
\end{itemize}
As in the proof of Theorem~\ref{th-mainresult} we denote by $R_{n,i}$ the regret resulting from
the selection of nodes in $\cT^i$, for $i \in \{ 1,2,3,4 \}$.

Lemma~\ref{lem:gooddomains-local} with $c = 2$ yields the bound $\E \bigl[ R_{n,1} \bigr]
\leq 4 L \nu_1 \rho^H n$,
where we crudely bounded by $n$ the number of times that nodes in $\cT^1$ were played.
Using that by definition each node of $\cT^2$ can be played only once,
we get
\[
\E \bigl[ R_{n,2} \bigr] \leq \sum_{h=z}^{H-1} \bigl( 4 L \nu_1 \rho^h \bigr) \, |\cI_{h}| \leq
4 C L \nu_1\nu_2^{-d}\,\sum_{h=z}^{H-1} \rho^{h(1-d)}\,.
\]
As for $R_{n,3}$, we also use here that nodes in $\cT^3$ belong to some $\cJ_h$, with $z+1 \leq h \leq H$;
in particular, they are the child of some element of $\cI_{h-1}$ and as such, firstly, they are
$4 L \nu_1 \rho^{h-1}$--optimal (by Lemma~\ref{lem:gooddomains-local})
and secondly, their number is bounded by $|\cJ_h| \leq 2 |\cI_{h-1}| \leq 2 C \bigl( \nu_2 \rho^{h-1} \bigr)^{-d}$.
Thus,
\[
\E \bigl[ R_{n,3} \bigr]
\leq \sum_{h=z+1}^H \bigl( 4 L \nu_1 \rho^{h-1} \bigr) \!\! \sum_{i : (h,i) \in \cJ_h} \!\! \E \bigl[ T_{h,i}(n) \bigr]
\leq 8 C L \nu_1 \nu_2^{-d} \sum_{h=z+1}^H \rho^{(h-1)(1-d)}\,\left(
\frac{8 \, \ln n}{\nu_1^2\rho^{2h}} + 4  \right)\,,
\]
where we used the bound of Lemma~\ref{lem:4-local}.
Finally, for $\cT^4$, we use that it contains at most $2^z - 1$ nodes, each of them being associated
with a regret controlled by Lemma~\ref{lem:4-local}; therefore,
\[
\E \bigl[ R_{n,4} \bigr] \leq \bigl( 2^{z} -1 \bigr) \left(\frac{8\, \ln n}{\nu_1^2 \rho^{2 z}} + 4\right)\,.
\]

\textbf{Fourth step.}
Putting things together, we have proved that
\[
\E \bigl[ R_n \bigr] \leq 4 L \nu_1 \rho^H n + \E \bigl[ R_{n,2} \bigr] + \E \bigl[ R_{n,3} \bigr]
+ \bigl( 2^{z} -1 \bigr) \left(\frac{8\, \ln n}{\nu_1^2 \rho^{2 z}} + 4\right)\,,
\]
where (using that $\rho < 1$ in the second inequality)
\begin{eqnarray*}
\lefteqn{ \E \bigl[ R_{n,2} \bigr] + \E \bigl[ R_{n,3} \bigr] } \\
& \leq & 4 C L \nu_1\nu_2^{-d}\,\sum_{h=z}^{H-1} \rho^{h(1-d)} +
8 C L \nu_1 \nu_2^{-d} \sum_{h=z+1}^H \rho^{(h-1)(1-d)}\,\left(
\frac{8 \, \ln n}{\nu_1^2\rho^{2h}} + 4  \right) \\
& = & 4 C L \nu_1\nu_2^{-d}\,\sum_{h=z}^{H-1} \rho^{h(1-d)} +
8 C L \nu_1 \nu_2^{-d} \sum_{h=z}^{H-1} \rho^{h(1-d)}\,\left(
\frac{8 \, \ln n}{\nu_1^2\rho^2\rho^{2h}} + 4  \right) \\
& \leq & 4 C L \nu_1\nu_2^{-d}\,\sum_{h=z}^{H-1} \rho^{h(1-d)} \, \frac{1}{\rho^{2h}} +
8 C L \nu_1 \nu_2^{-d} \sum_{h=z}^{H-1} \rho^{h(1-d)}\,\left(
\frac{8 \, \ln n}{\nu_1^2\rho^2\rho^{2h}} + \frac{4}{\rho^{2h}} \right) \\
& = & C L \nu_1 \nu_2^{-d} \left( \sum_{h=z}^{H-1} \rho^{-h(1+d)} \right)
\left( 36 + \frac{64}{\nu_1^2\rho^2} \, \ln n \right)\,.
\end{eqnarray*}
Denoting
\[
\gamma = \frac{4 \, C L \nu_1 \nu_2^{-d}}{(1/\rho)^{d+1} \, - 1} \left( \frac{16}{\nu_1^2\rho^2} + 9 \right)\,,
\]
it follows that for $n\ge 2$
\[
\E \bigl[ R_{n,2} \bigr] + \E \bigl[ R_{n,3} \bigr] \leq \gamma \, \rho^{-H(d+1)} \, \ln n\,.
\]

It remains to define the parameter $H \geq z$.
In particular, we propose to choose it such that the terms
\[
4 L \nu_1 \rho^H n \qquad \mbox{and} \qquad \rho^{-H(d+1)} \, \ln n
\]
are balanced. To this end, let $H$ be the smallest integer $k$ such that
$4 L \nu_1 \rho^k n \leq \gamma \rho^{-k(d+1)} \ln n$; in particular,
\[
\rho^H \leq \left( \frac{\gamma \ln n}{4 L \nu_1 n} \right)^{1/(d+2)}
\]
and
\[
4 L \nu_1 \rho^{H-1} n > \gamma \rho^{-(H-1)(d+1)} \ln n\,, \qquad
\mbox{implying} \qquad \gamma \, \rho^{-H(d+1)} \, \ln n \leq
4 L \nu_1 \rho^{H} n \,\, \rho^{-(d+2)}\,.
\]
Note from the inequality that this $H$ is such that
\[
H \geq \frac{1}{d+2} \frac{\ln(4 L \nu_1 n) - \ln(\gamma \ln n)}{\ln(1/\rho)}
\]
and thus this $H$ satisfies $H\ge z$ in view of the assumption of the theorem
indicating that $n$ is large enough.
The final bound on the regret is then
\begin{eqnarray*}
\E \bigl[ R_{n} \bigr] & \leq & 4 L \nu_1 \rho^H n + \gamma \, \rho^{-H(d+1)} \, \ln n +
\bigl( 2^{z} -1 \bigr) \left(\frac{8\, \ln n}{\nu_1^2 \rho^{2 z}} + 4\right) \\
& \leq & \left( 1 + \frac{1}{\rho^{d+2}} \right) 4 L \nu_1 \rho^H n
+ \bigl( 2^{z} -1 \bigr) \left(\frac{8\, \ln n}{\nu_1^2 \rho^{2 z}} + 4\right) \\
& \leq & \left( 1 + \frac{1}{\rho^{d+2}} \right) 4 L \nu_1 n
\left( \frac{\gamma \ln n}{4 L \nu_1 n} \right)^{1/(d+2)}
+ \bigl( 2^{z} -1 \bigr) \left(\frac{8\, \ln n}{\nu_1^2 \rho^{2 z}} + 4\right) \\
& = & \left( 1 + \frac{1}{\rho^{d+2}} \right) \bigl( 4 L \nu_1 n \bigr)^{(d+1)/(d+2)}
( \gamma \ln n)^{1/(d+2)}
+ \bigl( 2^{z} -1 \bigr) \left(\frac{8\, \ln n}{\nu_1^2 \rho^{2 z}} + 4\right)\,.
\end{eqnarray*}
This concludes the proof.
\end{proof}

\subsection{Proof of Theorem~\ref{th:localHOO} (regret bound for local-HOO)}
\label{sec:thlocalHOO}

\begin{proof}
We use the notation of the proof of Theorem~\ref{th:zHOO}.
Let $r_0$ be a positive integer such that for $r \geq r_0$, one has
\[
z_r \eqdef \lceil \log_2 r \rceil \geq h_0 \qquad \mbox{and} \qquad
z_r \leq \frac{1}{d+2} \frac{\ln(4 L \nu_1 2^r) - \ln( \gamma \ln 2^r)}{\ln(1/\rho)}~;
\]
we can therefore apply the result of Theorem \ref{th:zHOO} in regimes indexed by $r \geq r_0$.
For previous regimes, we simply upper bound the regret by the number of rounds,
that is, $2^{r_0} - 2 \leq 2^{r_0}$.
For round $n$, we denote by $r_n$ the index of the regime where
$n$ lies in (regime $r_n = \lfloor \log_2 (n+1) \rfloor$).
Since regime $r_n$ terminates at round $2^{r_n+1}-2$,
we have
\begin{eqnarray*}
\lefteqn{ \E \bigl[ R_n \bigr] \ \leq \ \E \bigl[ R_{2^{r_n+1}-2} \bigr] } \\
& \leq & 2^{r_0} + \sum_{r = r_0}^{r_n} \Biggl( \left( 1 + \frac{1}{\rho^{d+2}} \right) \bigl( 4 L \nu_1 2^r \bigr)^{(d+1)/(d+2)}
( \gamma \ln 2^r)^{1/(d+2)}
+ \bigl( 2^{z_r} -1 \bigr) \left(\frac{8\, \ln 2^r}{\nu_1^2 \rho^{2 z_r}} + 4\right) \Biggr) \\
& \leq & 2^{r_0} + C_1 \, (\ln n) \, \sum_{r = r_0}^{r_n} \biggl( \Bigl( 2^{(d+1)/(d+2)} \Bigr)^r
+ \bigl( 2 / \rho^2 \bigr)^{z_r} \biggr)  \\
& \leq & 2^{r_0} + C_2 \, (\ln n) \, \biggl( \Bigl( 2^{(d+1)/(d+2)} \Bigr)^{r_n}
+ r_n \, \bigl( 2/\rho^2 \bigr)^{z_{r_n}} \biggr) \ \ = (\ln n) \,\, O \bigl( n^{(d+1)/(d+2)} \bigr)\,,
\end{eqnarray*}
where $C_1,C_2>0$ denote some constants depending only on the parameters but not on $n$.
Note that for the last equality we used that the first term
in the sum of the two terms that depend on $n$ dominates the second term.
\end{proof}

\subsection{Proof of Theorem~\ref{th:minimaxUB} (uniform upper bound on
the regret of HOO against the class of all weak Lipschitz environments)}
\label{sec:thminimaxUB}

Equations~(\ref{eq:WL1}) and~(\ref{eq:WL2}), which follow from Assumption~\ref{ass:weaklip},
show that Assumption~\ref{ass:weaklipv} is satisfied for $L=2$ and all $\epsilon_0 > 0$. We take,
for instance, $\eps_0 = 3 \nu_1$.
Moreover, since $\cX$ has a packing dimension of $D$, all environments have a near-optimality dimension less than $D$.
In particular, for all $D'>D$ (as shown in the second step of the proof of
Theorem~\ref{th-mainresult} in Section~\ref{sec:proof1}),
there exists a constant $C$ (depending only on $\ell$, $\cX$, $\eps_0 = 3\nu_1$, $\nu_2$, and $D'$)
such that Assumption~\ref{ass:a3} is satisfied.
We can therefore take $h_0 = 0$ and apply Theorem~\ref{th:zHOO} with $z=0$ and $M \in \cF_{\cX,\ell}$;
the fact that all the quantities involved in the bound depend only
on $\cX$, $\ell$, $\nu_2$, $D'$,
and the parameters of HOO, but not on a particular environment in $\cF$, concludes the proof.

\subsection{Proof of Theorem \ref{th:minimaxLB} (minimax lower bound in metric spaces)}

Let $K\geq 2$ an integer to be defined later. We provide first an overview of the proof.
Here, we exhibit a set $\cA$ of environments for the $\{1,\hdots,K+1\}$--armed bandit problem
and a subset $\cF' \subset \cF_{\cX,\ell}$ which satisfy the following properties.
\begin{description}
\item[(i)] The set $\cA$ contains ``difficult'' environments for the $\{1,\hdots,K+1\}$--armed bandit problem.
\item[(ii)] For any strategy $\cXstrat$ suited to the $\cX$--armed bandit problem,
one can construct a strategy $\Kstrat$ for the $\{1,\hdots,K+1\}$--armed bandit problem such that
\[
\forall \,\, M \in \cF', \ \ \exists \,\, \nu \in \cA, \qquad \E_M \bigl[ R_n(\cXstrat) \bigr] = \E_{\nu} \bigl[
R_n(\Kstrat) \bigr]\,.
\]
\end{description}
We now provide the details. \\

\begin{proof}
We only deal with the case of deterministic strategies. The extension to randomized strategies
can be done using Fubini's theorem (by integrating also \wrt\  the auxiliary randomizations
used).

\textbf{First step.}
Let $\eta \in (0,1/2)$ be a real number and $K \geq 2$ be an integer, both to be defined during the course of the analysis.
The set $\cA$ only contains $K$ elements, denoted by $\nu^1,\ldots,\nu^K$ and given by product distributions.
For $1 \leq j \leq K$, the distribution $\nu^j$ is obtained as the product of the $\nu^j_i$
when $i \in \{1, \ldots, K+1 \}$ and where
\[
\nu_i^j =
\begin{cases}
\Ber(1/2), & \mbox{if} \ i \ne j; \\
\Ber(1/2 + \eta), & \mbox{if} \ i = j.
\end{cases}
\]
One can extract the following result
from the proof of the lower bound of~\cite[Section~6.9]{CL06}.
\begin{lemma} \label{lemma-multiarmedLB1}
For all strategies $\Kstrat$ for the $\{1,\hdots,K+1\}$--armed bandit (where
$K \geq 2$), one has
\[
\max_{j = 1, \hdots, K}  \,\, \E_{\nu^j} \bigl[ R_n(\Kstrat) \bigr]
\geq n \eta \left(1 - \frac{1}{K} - \eta \sqrt{4 \ln (4/3)} \sqrt{\frac{n}{K}} \right)\,.
\]
\end{lemma}

\textbf{Second step.}
We now need to construct $\cF'$ such that item (ii) is satisfied.
We assume that $K$ is such that $\cX$ contains $K$ disjoint
balls with radius $\eta$. (We shall quantify later in this proof a
suitable value of $K$.) Denoting by $x_1,\ldots,x_K$
the corresponding centers, these disjoint balls are then $\cB(x_1,\eta), \, \ldots, \,
\cB(x_K,\eta)$.

With each of these balls we now associate a bandit environment over $\cX$, in
the following way. For all $x^* \in \cX$, we introduce a mapping
$g_{x^*,\eta}$ on $\cX$ defined by
\[
g_{x^*,\eta}(x) = \max \bigl\{ 0, \,\, \eta - \ell(x,x^*) \bigr\}
\]
for all $x \in \cX$. This mapping is used to define an environment
$M_{x^*,\eta}$ over $\cX$, as follows. For all $x \in \cX$,
\[
M_{x^*,\eta}(x) = \Ber \left( \frac{1}{2} + g_{x^*,\eta}(x) \right)\,.
\]
Let $f_{x^*,\eta}$ be the corresponding mean-payoff function; its values equal
\[
f_{x^*,\eta}(x) = \frac{1}{2} + \max \bigl\{ 0, \,\, \eta - \ell(x,x^*) \bigr\}
\]
for all $x \in \cX$.
Note that the mean payoff is maximized at $x = x^*$ (with value $1/2+\eta$)
and is minimal for all points lying outside $\cB(x^*,\eta)$, with value $1/2$.
In addition, that $\ell$ is a metric entails
that these mean-payoff functions are $1$--Lipschitz
and thus are also weakly Lipschitz. (This is the only point
in the proof where we use that $\ell$ is a metric.)
In conclusion, we consider
\[
\cF' = \bigl\{ M_{x_1,\eta}, \, \ldots, \, M_{x_K,\eta} \bigr\} \,\, \subset \cF_{\cX,\ell}\,.
\]

\textbf{Third step.}
We describe how to associate with each (deterministic) strategy $\cXstrat$ on $\cX$
a (random) strategy $\Kstrat$ on the finite set of arms $\{ 1, \ldots, K+1 \}$.
Each of these strategies is indeed given by a sequence of mappings,
\[
\cXstrat_1,\cXstrat_2,\ldots \qquad \mbox{and} \qquad
\Kstrat_1,\Kstrat_2,\ldots
\]
where for $t \geq 1$, the mappings $\cXstrat_t$ and $\Kstrat_t$ should only depend on
the past up to the beginning of round $t$. Since the strategy $\cXstrat$ is deterministic,
the mapping $\cXstrat_t$ takes only into account the past rewards $Y_1,\ldots,Y_{t-1}$ and is therefore
a mapping $[0,1]^{t-1} \to \cX$. (In particular, $\cXstrat_1$ equals a constant.)

We use the notations $I'_t$ and $Y'_t$ for, respectively, the
arms pulled and the rewards obtained by the strategy $\Kstrat$ at each round $t$.
The arms $I'_t$ are drawn at random according to the distributions
\[
\Kstrat_t \bigl( I'_1,\ldots,I'_{t-1}, \, Y'_1,\ldots,Y'_{t-1} \bigr)\,,
\]
which we now define. (Actually, they will depend on the
obtained payoffs $Y'_1,\ldots,Y'_{t-1}$ only.)
To do that, we need yet another mapping $T$ that links
elements in $\cX$ to probability distributions over $\{ 1,\ldots,K+1\}$.
Denoting by $\delta_k$ the Dirac probability on $k \in \{ 1,\ldots,K+1\}$,
the mapping $T$ is defined as
\[
T(x) =
\begin{cases}
\delta_{K+1}\,, & \text{if } \ \displaystyle{x \not\in \bigcup_{j=1,\ldots,K} \cB(x_j,\eta)}; \\
\displaystyle{\left( 1- \frac{\ell(x,x_j)}{\eta} \right) \, \delta_{j} +
\frac{\ell(x,x_j)}{\eta} \, \delta_{K+1}}\,, & \text{if } \ x \in \cB(x_j,\eta) \ \mbox{for some} \
j \in \{ 1,\ldots,K \},
\end{cases}
\]
for all $x \in \cX$.
Note that this definition is legitimate because the balls $\cB(x_j,\eta)$
are disjoint when $j$ varies between $1$ and $K$.

Finally, $\Kstrat$ is defined as follows. For all $t \geq 1$,
\[
\Kstrat_t \bigl( I'_1,\ldots,I'_{t-1}, \, Y'_1,\ldots,Y'_{t-1} \bigr)
= \Kstrat_t \bigl( Y'_1,\ldots,Y'_{t-1} \bigr)
= T \Bigl( \cXstrat_t \bigl( Y'_1,\ldots,Y'_{t-1} \bigr) \Bigr)\,.
\]

Before we proceed, we study the distribution of the reward $Y'$ obtained
under $\nu^i$ (for $i \in \{ 1,\ldots,K\}$)
by the choice of a random arm $I'$ drawn according to $T(x)$, for some $x \in \cX$.
Since $Y'$ can only take the values 0 or 1,
its distribution is a Bernoulli distribution
whose parameter $\mu_i(x)$ we compute now.
The computation is based on the fact that under $\nu^i$,
the Bernoulli distribution corresponding to arm $j$
has $1/2$ as an expectation, except if $j = i$, in which case it is $1/2+\eta$. Thus,
for all $x \in \cX$,
\[
\mu_i(x) =
\begin{cases}
1/2\,, & \mbox{if} \ x \not\in \cB(x_i,\eta); \vspace{.15cm} \\
\displaystyle{\left( 1- \frac{\ell(x,x_i)}{\eta} \right) \, \left( \frac{1}{2} + \eta \right) +
\frac{\ell(x,x_i)}{\eta} \, \frac{1}{2}
= \frac{1}{2} + \eta - \ell(x,x_i)}\,, & \mbox{if} \ x \in \cB(x_i,\eta).
\end{cases}
\]
That is, $\mu_i = f_{x_i,\eta}$ on $\cX$.

\textbf{Fourth step.}
We now prove that the distributions of the regrets of $\cXstrat$ under $M_{x_j, \eta}$
and of $\Kstrat$ under $\nu^j$ are equal for all $j = 1,\ldots,K$.
On the one hand, the expectations of rewards associated with the best arms equal $1/2+\eta$ under
the two environments.
On the other hand, one can prove by induction that the sequences $Y_1,Y_2,\ldots$ and $Y'_1,Y'_2,\ldots$
have the same distribution. (In the argument below, conditioning
by empty sequences means no conditioning. This will be the case only for $t=1$.)

For all $t \geq 1$, we denote
\[
X'_t = \cXstrat_t \bigl( Y'_1,\ldots,Y'_{t-1} \bigr)\,.
\]
Under $\nu^j$ and given $Y'_1,\ldots,Y'_{t-1}$, the distribution of $Y'_t$
is obtained by definition as the two-step random draw of $I'_t \sim T(X'_t)$
and then, conditionally on this first draw,
$Y'_t \sim \nu^j_{I'_t}$. By the above results, the distribution of $Y'_t$
is thus a Bernoulli distribution with parameter $\mu_j(X'_t)$.

At the same time,
under $M_{x_j, \eta}$ and given $Y_1,\ldots,Y_{t-1}$,
the choice of
\[
X_t = \cXstrat_t \bigl( Y_1,\ldots,Y_{t-1} \bigr)
\]
yields a reward $Y_t$ distributed according to
$M_{x_j, \eta}(X_t)$, that is, by definition and with the notations above,
a Bernoulli distribution with parameter
$f_{x_j,\eta}(X_t) = \mu_j(X_t)$.

The argument is concluded by induction and by
using the fact that rewards are drawn independently in each round.

\textbf{Fifth step.}
We summarize what we proved so far.
For $\eta \in (0,1/2)$, provided that there exist $K \geq 2$ disjoint balls $\cB(x_j,\eta)$ in $\cX$,
we could construct, for all strategies $\cXstrat$ for the $\cX$--armed bandit problem,
a strategy $\Kstrat$ for the $\{1,\ldots,K+1\}$--armed bandit problem such that, for all
$j = 1,\ldots,K$ and all $n \geq 1$,
\[
\E_{M_{x_j,\eta}} \bigl[ R_n(\cXstrat) \bigr] = \E_{\nu^j} \bigl[ R_n(\Kstrat) \bigr]\,.
\]

But by the assumption on the packing dimension, there exists $c > 0$ such that for all $\eta < 1/2$,
the choice of $K_\eta = \lceil c\,\eta^{-D} \rceil \geq 2$ guarantees the existence
of such $K_\eta$ disjoint balls.
Substituting this value, and using the results of the first and fourth steps of the proof,
we get
\[
\max_{j=1,\ldots,K_\eta} \,\, \E_{M_{x_j,\eta}} \bigl[ R_n(\cXstrat) \bigr] =
\max_{j=1,\ldots,K_\eta} \,\, \E_{\nu^j} \bigl[ R_n(\Kstrat) \bigr]
\geq n \eta \left(1 - \frac{1}{K_\eta} - \eta \sqrt{4 \ln (4/3)} \sqrt{\frac{n}{K_\eta}} \right)\,.
\]
The proof is concluded by noting that
\begin{itemize}
\item the left-hand side is smaller than the maximal regret \wrt\ all weak Lipschitz
environments;
\item the right-hand side can be lower bounded and then
optimized over $\eta < 1/2$ in the following way.
\end{itemize}
By definition of $K_\eta$ and the fact that it is larger than 2, one has
\begin{multline}
\nonumber
n \eta \left(1 - \frac{1}{K_\eta} - \eta \sqrt{4 \ln (4/3)} \sqrt{\frac{n}{K_\eta}} \right) \\
\geq n \eta \left(1 - \frac{1}{2} - \eta \sqrt{4 \ln (4/3)} \sqrt{\frac{n}{c \eta^{-D}}} \right)
= n \eta \left(\frac{1}{2} - C \, \eta^{1+D/2} \sqrt{n} \right)
\end{multline}
where $C = \sqrt{\bigl(4 \ln (4/3) \bigr) \, \big/ \, c}$. We can optimize the final lower bound over $\eta \in [0,\,1/2]$.

To that end, we choose, for instance, $\eta$ such that $C \, \eta^{1+D/2} \sqrt{n} = 1/4$, that is,
\[
\eta = \left( \frac{1}{4 C \sqrt{n}} \right)^{1/(1+D/2)} =
\left( \frac{1}{4 C} \right)^{1/(1+D/2)} \, n^{-1/(D+2)}\,.
\]
This gives the lower bound
\[
\frac{1}{4} \, \left( \frac{1}{4 C} \right)^{1/(1+D/2)} \, n^{1-1/(D+2)}
= \underbrace{\frac{1}{4} \, \left( \frac{1}{4 C} \right)^{1/(1+D/2)}}_{= \,\, \gamma(c,D)} \,\, n^{(D+1)/(D+2)}\,.
\]
To ensure that this choice of $\eta$ is valid we need to show that $\eta \leq 1/2$.
Since the latter requirement is equivalent to
\[
n \geq \left( 2 \left( \frac{1}{4 C} \right)^{1/(1+D/2)} \right)^{D+2}\,,
\]
it suffices to choose the right-hand side to be $N(c,D)$;
we then get that $\eta\leq 1/2$ indeed holds for all $n \geq N(c,D)$, thus concluding the proof of the theorem.
\end{proof}

\bibliography{biblio}
\end{document}